\def\year{2022}\relax
\documentclass[letterpaper]{article} 
\usepackage{aaai22}  
\usepackage{times}  
\usepackage{helvet}  
\usepackage{courier}  
\usepackage[hyphens]{url}  
\usepackage{graphicx} 
\usepackage{natbib}  
\usepackage{caption} 
\DeclareCaptionStyle{ruled}{labelfont=normalfont,labelsep=colon,strut=off} 
\nocopyright

\title{Hierarchical Reinforcement Learning with AI Planning Models}
\author{
    Junkyu Lee, 
    Michael Katz, 
    Don Joven Agravante, \\
    Miao Liu, 
    Geraud Nangue Tasse,
    Tim Klinger, 
    Shirin Sohrabi\\
}

\affiliations{IBM Research AI\\
\{Junkyu.Lee,michael.katz1,Don.Joven.R.Agravante\}@ibm.com,\\
geraudnt@gmail.com, 
\{tklinger,ssohrab\}@us.ibm.com
}

\usepackage{amsmath}
\usepackage{booktabs}
\usepackage{latexsym} 
\usepackage{amssymb} 
\usepackage{multirow}
\usepackage{stmaryrd}
\usepackage{enumerate}

\usepackage{algorithm} 
\usepackage{algorithmicx}
\usepackage[noend]{algpseudocode}
\algrenewcommand\alglinenumber[1]{\scriptsize #1:}

\usepackage{color}
\usepackage{tikz}
\usepackage{caption}
\usepackage{subcaption}
\usetikzlibrary{positioning}
\usetikzlibrary{automata,arrows}
\usetikzlibrary{shapes,decorations}

\def\defemb#1#2{\expandafter\def\csname #1\endcsname
                              {\relax\ifmmode #2\else\hbox{$#2$}\fi}}

                              \defemb{cO}{{\cal O}}
                              \defemb{cT}{{\cal T}}

\newtheorem{theorem}{Theorem}
\newtheorem{definition}{Definition}

\newcommand{\sasplus}{\textsc{sas}^+}
\newcommand{\var}{v}
\newcommand{\vars}{\ensuremath{\mathcal V}}
\newcommand{\ops}{\ensuremath{\mathcal O}}

\newcommand{\transprob}{P}
\newcommand{\transreward}{r}

\newcommand{\options}{\cO}
\newcommand{\option}{{\ensuremath{O}}}
\newcommand{\goaloption}{\option_{\ast}}

\newcommand{\action}{a}
\newcommand{\op}{o}
\def\actions{\mathcal{A}}
\newcommand{\state}{s}
\def\states{\mathcal{S}}
\newcommand{\initstate}{s_0'}
\newcommand{\goal}{s_{\ast}}
\newcommand{\Goal}{G}

\newcommand{\domain}{\ensuremath{\mathit{dom}}}

\newcommand{\fact}[1]{\ensuremath{\langle \var, #1 \rangle}}
\newcommand{\val}{\vartheta}
\newcommand{\prv}{\ensuremath{\textit{prv}}}
\newcommand{\pre}{\ensuremath{\textit{pre}}}
\newcommand{\eff}{\ensuremath{\mathit{eff}}}

\newcommand{\leftq}{\llbracket}
\newcommand{\rightq}{\rrbracket}
\newcommand{\applied}[1]{\leftq #1 \rightq}
\newcommand{\plan}{\pi}

\newcommand{\ptask}{\Pi}
\newcommand{\mdptask}{{\ensuremath{\mathcal M}}}
\newcommand{\initset}[1]{{\ensuremath{\mathcal I}_{#1}}}
\newcommand{\termset}[1]{\beta_{#1}}

\newcommand{\tuple}[1]{\langle #1 \rangle}

\newenvironment{proof}{\noindent {\bf Proof:}}%
 {\hfill \rule[0.3ex]{1ex}{1ex} \par \addvspace{\bigskipamount}}

 \newcommand{\statesmap}{L}

\makeatother

\newcommand{\cC}{\mathcal{C}}
\newcommand{\cD}{\mathcal{D}}

\newcommand{\cF}{\mathcal{F}}

\begin{document}

\maketitle

\begin{abstract}
Two common approaches to sequential decision-making are AI planning (AIP) and reinforcement learning (RL). Each has strengths and weaknesses. AIP is interpretable, easy to integrate with symbolic knowledge, and often efficient, but requires an up-front logical domain specification and is sensitive to noise; RL only requires specification of rewards and is robust to noise but is sample inefficient and not easily supplied with external knowledge. We propose an integrative approach that combines high-level planning with RL, retaining interpretability, transfer, and efficiency, while allowing for robust learning of the lower-level plan actions.

Our approach defines options in hierarchical reinforcement learning (HRL) from AIP operators by establishing a correspondence between the state transition model of AI planning problem and the abstract state transition system of a Markov Decision Process (MDP). Options are learned by adding intrinsic rewards to encourage consistency between the MDP and AIP transition models.
We demonstrate the benefit of our integrated approach by comparing the performance of RL and HRL algorithms in both MiniGrid and N-rooms environments, showing the advantage of our method over the existing ones.
\end{abstract}

\section{Introduction}
Sequential decision-making problems have been historically tackled with two distinct and largely complementary research paradigms: AI planning (AIP) and reinforcement learning (RL). 
In AIP, a human modeler creates a domain specification in logic which specifies action (operator) preconditions and effects. This approach can yield computationally efficient planning and interpretable plans, 
but AIP planners are not tolerant to noise or uncertainty,
and specification of the domain can be difficult when it is not well-understood or complex. 
By contrast, model-free deep reinforcement learning approaches lift the burden of model specification, inherit the tolerance to noise and uncertainty of neural networks, and require no special human understanding of the domain. But RL does not generate easily interpretable policies and domain-specific knowledge, which can be hugely important in sample efficient learning, must be specified obliquely through the design of training algorithms, policy networks and reward structures.

HRL \cite{barto2003recent} 
aims to improve the sample efficiency of non-hierarchical RL methods for solving large scale problems by exploiting domain knowledge that allows decomposing task structures in the abstract state and action space \cite{dean1995decomposition}.
Notable earlier works include 
hierarchical abstract machines \cite{parr1998reinforcement} 
that encode domain knowledge about high-level state transition into a finite state machine,
options framework \cite{sutton-et-al-aij1999} 
that characterizes HRL as executing sub-routines, each generating temporarily extended actions in semi-MDP (SMDP),
and feudal RL \cite{dayan1992feudal} 
that presents hierarchical control architecture,
where the lower-level agents achieve subgoals directed by the higher-level control agent. 
Then, later works focus on discovering temporally extended actions 
such as options \cite{bacon-et-al-aaai2017,machado-bellemare-bowling-icml2017,bagaria2019option},
learning state abstractions
\cite{RavindranApproximateH,li2006towards}
or
skills or sub-tasks \cite{mcgovern2001automatic, stolle-precup-isara2002, castro2011automatic, simsek2008skill}. More recently, state abstraction has also been used in option learning\cite{cradol} 
to further improve the sample efficiency of HRL. 
In addition to infusing knowledge for decomposing the task, 
HRL agents could inform the lower-level control agents 
with intrinsic rewards \cite{singh2005intrinsically} 
to better guide optimization procedure \cite{vezhnevets2017feudal,kulkarni2016hierarchical,nachum2018data}. 

Recently, 
we see increasing interest in 
integrating symbolic methods in AIP and deep RL due to their complementary nature.
Reward machines (RM) \cite{icarte2018using} specify the reward of MDP 
over a finite-state machine (FSM) 
such that the agent can learn policies that follow the symbolic event models,
encoded manually, or
translated from linear temporal logic (LTL) expressions specifying possible symbolic policies \cite{camacho2019ltl}. 
Then, those FSMs augmented with symbolic knowledge can also be utilized to define temporarily extended actions for HRL agents.
\cite{icarte2022reward,araki2021logical,den2022reinforcement}.
High-level instructions for guiding RL 
can also be provided 
through a sequence of symbolic trajectories
in various forms
such as state predicates and action operators in AIP \cite{illanes-et-al-icaps2020},
or other formal action languages \citet{yang-et-al-ijcai2018,lyu-et-al-aaai2019,KokelMNRT21,kokel2021deep}.

When the problem involves complex task structures, e.g.,
as in combined task and motion planning 
\cite{eppe2019semantics,garrett2021integrated}
or in the RL environments originated from AIP domains
\cite{toyer-et-al-aaai2018,groshev2018learning,shen-et-al-icaps2020},
the integrated approach is a more natural choice,
and we often have access to domain knowledge 
that captures the task structure for defining AIP models.
In this paper, 
we present an integrated AI planning and RL framework for HRL,
which we call Planning annotated RL (PaRL).
In PaRL, we provide an AIP model annotating the RL environment, 
which offers abstract and partial knowledge about the RL MDP. 
Unlike other integrated methods, 
the annotating AIP model is a valid planning task that the planning agent can supply to AI planners.
The main contributions of the paper are summarized as follows:
(1) 
Unlike other approaches that require 
a manual process or rely on the solution to the problem, 
we present a method for deriving options directly from AIP model. 
(2) We design a method for generating intrinsic rewards for RL agents 
that encourages consistency between the annotating task and MDP transitions by introducing consistency constraints at the planning level.
(3) We show  the improvement in sample efficiency due to decomposition and 
additional benefits inherit from AIP and RL approaches in 
MiniGrid and N-rooms environments \cite{gym_minigrid,babyai_iclr19}.

\section{Background}

\subsection{RL and Options Framework}
We assume that an agent interacts with 
a goal-oriented MDP 
$\mdptask =\langle \states, \actions, \transprob, \transreward, s_0, \Goal, \gamma\rangle$
with 
states $\states$, 
actions $\actions$, 
a state transition function $\transprob: \states \times  \actions \times \states \rightarrow [0, 1]$, 
a reward function $\transreward:\states \times \actions \rightarrow \mathbb{R}$,
an initial state $s_0\in\states$,
a set of goal states $\Goal\subset\states$,
and a discounting factor $\gamma \in (0, 1)$ for the rewards.
In this goal-oriented environment, 
we are interested in the sparse reward task,
and
the objective is to learn a stationary optimal policy $\pi^{*}$ 
that maximizes the expected return,
$\pi^* = \arg\max_{\pi} \mathbb{E}_{\pi} [ \sum_{t=0}^{\infty} \gamma^t r_t | s_0],$
where 
$s_0$ is the initial state, and
$\pi(a|s)$ is a stochastic policy $\pi: \states \times \actions \rightarrow [0, 1]$.

A value function $V^{\pi}(s)$
is the expected sum of the discounted reward in each state $s \in \states$,
$$V^{\pi}(s)\!\!=\!\!\sum_{a \in \actions} \pi(a|s) 
[ \transreward(s,a) + \gamma \sum_{s' \in \states} \transprob(s' | s,a) V^{\pi}(s')].
$$
The action-value function
gives
the value of executing an action $a \in \actions$ in state $s \in \states$ under the policy $\pi$,
$$Q^{\pi}(s,a) = \transreward(s,a) + \gamma \sum_{s' \in \states} \transprob(s'|s,a) Q^{\pi}(s').$$
The optimal value function $V^*(s)$ and action-value function $Q^*(s,a)$
can be found by 
$V^*(s)=\max_{\pi} V^{\pi}(s)$
and
$Q^*(s, a)=\max_{\pi} Q^{\pi}(s, a).$

In options framework \cite{sutton-barto-1998},
A set of options $\options$ formalizes
the temporally extended actions
that defines a semi-MDP (SMDP) over the original MDP $\mdptask$.
A Markovian option $\option \in \options$ is 
a triple $\langle\initset{\option}, \pi_{\option}, \termset{\option}\rangle$,
where
$\initset{\option}$ is the initiation set in which $\option$ can begin,
$\pi_{\option}$ is a stationary option policy 
$\pi_{\option}: \states \times \actions \rightarrow [0, 1]$,
and
$\termset{\option}$ is a termination set in which $\option$ terminates.
We follow the call-and-return option execution model, 
where an agent selects an option $\option$ using 
an option level policy  $\mu(\option|\state)$ in state $s$ at time $t$,
and generates a sequence of actions according to the option policy  $\pi_o(a|s)$. 
The execution of an option $\option$ continues up to $k$ steps until reaching the $\beta_{\option}$ and 
it returns the option reward $R(s, \option)$
accumulated from $t+1$ to $t+k$ with a discounting factor $\gamma$,
$$
R(s, \option) = \mathbb{E}\big[\sum_{t'=t+1}^{t+k} \gamma^{t'-t-1} r_{t'}|\mathcal{E}(\option, s, t)\big],
$$
where $\mathcal{E}(\option, s, t)$ denotes the event of 
an option $\option$ being selected in state $s$ at time $t$,
and $r_{t'}$ denotes the reward obtained at time $t'$.
The state transition probability 
from a state $s$ to a state $s'$ under the execution of an option $\option$ can be written as
$$
\transprob(s'|s,\option)=\sum_{j=0}^{\infty} \gamma^{j} Pr\big(k=j, s_{t+j} | \mathcal{E}(\option, s, t)\big).
$$

In SMDP,
the value function 
$V^{\mu}(s)$ under the option level policy $\mu$ 
can be written as
$$
V^{\mu}(s)\!=\!\sum_{\option \in \options}
\mu(\option|s) \Big[R(s,\option)\!+\!\sum_{s' \in \states}
P(s'|s,\option)V^{\mu}(s') \Big],
$$
and 
the option-value function $Q^{\mu}(s, O)$ is
$$
Q^{\mu}(s,\option)\!=\!R(s,\option)\!+\!\sum_{s' \in \states}\!\!
P(s'|s,\option) \sum_{\option \in \options}\!\!\mu(\option|s) Q^{\mu}(s', \option).
$$
In general, learning options ranges 
from the offline option discovery to the online end-to-end option critic approach \cite{bacon-et-al-aaai2017},
and 
each option policy $\pi_{\cO}$ could be trained by 
existing RL algorithms such as value-based methods or policy-gradient methods.
Given a set of learned options $\options$, 
an off-policy learning methods such as Q-learning \cite{watkins-dayan-ml1992,mnih2015human}
can learn the option value function 
by SMDP Q-learning \cite{sutton-et-al-aij1999}.  

\subsection{AI Planning} 
To formally represent planning tasks, we follow the notation of
$\sasplus$ planning tasks \cite{backstrom-nebel-compint1995}.
In $\sasplus$, 
a planning task
$\ptask$ is given by a tuple 
$\langle\vars,\ops,\initstate, \goal\rangle$, 
where 
$\vars$ is a finite set of state variables,
and $\ops$ is a finite set of operators.
Each state variable $\var \in
\vars$ has a finite domain $\domain(\var)$ of values. A pair $\fact{\val}$ with
$\var \in \vars$ and $\val \in \domain(\var)$ is called a fact. 
A (partial) assignment to $\vars$ is called a (partial) state, 
with 
the full state $\initstate$ being the initial state
and 
the partial state $\goal$ being the goal. 
We denote the variables of a partial assignment $p$ by $\vars(p)$. 
It is convenient to view a partial state $p$ as a
set of facts with $\langle \var, \val \rangle \in p$ if and only if $p[\var] =\val$. 
A partial state $p$ is consistent with state $\state$ if $p\subseteq s$. 

We denote the set of states of $\ptask$ by $\states'$. 
Each operator $\op\!\in\!\ops$ is a pair 
$\langle \pre(\op),\eff(\op)\rangle$ of partial states called 
preconditions and effects. 
The (possibly empty) subset
of preconditions that do not involve variables from the effect is
called prevail condition, 
$\prv(\op)\!=\!\{\langle\!\var,\val\!\rangle\!\mid\!\langle \var, \val \rangle\!\in\!\pre(\op), \var \not\in\vars(\eff(\op)) \}$.
An operator $\op$ is applicable in a state $\state\in\states'$ if and only if
$\pre(\op)$ is consistent with $\state$ ($\pre(\op)\subseteq\state$). 
Applying $\op$ changes the
value of $\var$ to $\eff(\op)[\var]$, if defined. 
The resulting state is
denoted by $\state\applied{\op}$. 
An operator sequence 
$\plan =\langle\op_{1},\dots,\op_{k}\rangle$ is applicable in $\state$ 
if there exist states $\state_0, \cdots, \state_k$ such that 
(1) $\state_0 = \state$, and
(2) for each $1 \leq i \leq k$, $\pre(\op_i)\subset\state_{i\mbox{-}1}$ and
$\state_{i} = \state_{i\mbox{-}1}\applied{\op_i}$. 
We denote the state $\state_k$ by $\state\applied{\plan}$. 
$\plan$ is a plan for $s$ iff $\plan$ is applicable in
$s$ and $\goal\subseteq s\applied{\plan}$. 

A transition graph of a planning task $\ptask=\langle\vars,\ops,\initstate,\goal\rangle$ is
a triple $\cT_{\ptask} = \tuple{\states, T_{\ptask}, \goal}$, 
where 
$\states$ are the states of $\ptask$, 
$T_{\ptask} \subseteq \states\times\ops\times\states$ is 
a set of labeled transitions, 
and $\goal \subseteq \states$ is the set of goal states.
An abstraction of the transition graph $\cT$ is a pair $\tuple{\cT',\alpha}$,
where $\cT' = \tuple{\states', T', \states_{\ast}'}$ is an abstract transition graph 
and $\alpha: \states\mapsto\states'$ is an abstraction mapping,
such that $\tuple{\alpha(s),\op,\alpha(s')}\in T'$ for all $\tuple{s,\op,s'}\in
T$, and $\alpha(s)\in \states_{\ast}'$ for all $s\in \states_{\ast}$.

\section{Annotating RL with Planning}
In this section,
we formulate our HRL framework.
The basic idea is to link 
the AI planning task and the 
MDP task 
by viewing 
the former as an abstraction of the latter
and mapping 
all transitions associated with a planning operator
to a temporal abstraction encapsulated in the RL option \cite{sutton-et-al-aij1999}.

\subsection{PaRL Task}
We start by defining a Planning annotated RL (PaRL) task 
and present the options framework derived from a symbolic planning task.

\begin{definition}
A \textbf{PaRL task} is a triple 
$E:=\tuple{\mdptask, \ptask, \statesmap}$, 
where 
$\mdptask:=\langle
\states, \actions, \transprob, \transreward, s_0, \Goal, \gamma\rangle
$ is a goal-oriented MDP over states $\states$, 
$\ptask:=\langle 
\vars, \ops, \initstate, \goal
\rangle$ is a planning task over 
states $\states'$, and
$\statesmap: \states \mapsto \states'$ is a surjective mapping 
from the MDP states $\states$ to planning task states $\states'$
satisfying 
$\initstate=\statesmap(s_0)$ and 
$\goal$ consistent with $\statesmap(\state)$ for all $\state\!\!\in\!\!\Goal$.
We denote the pre-image of $\state'\!\!\in\!\!\states'$ under $\statesmap$,
$\{s\!\!\in\!\!\states\!\!\mid\!\!\statesmap(\state)\!\!=\!\!\state'\}$
by $\statesmap^{-1}(\state')$.
\end{definition}

The generic definition of PaRL task is a mixed blessing. 
On the one hand, it does not pose any constraints on the
connection between $\mdptask$ and $\ptask$ 
beyond the consistency of 
the initial state and the goal under $\statesmap$.
On the other hand, if the two tasks are unrelated, it is
not clear what is the benefit of connecting these tasks together. 
We formulate the connection by extending the definition of abstraction to PaRL tasks. 

\begin{definition}
\label{def:abs}
Let $E = \tuple{\mdptask, \ptask, \statesmap}$ be a PaRL task and
$\cT_{\ptask}=\tuple{\states', T_{\ptask}, S_{\ast}}$ be the transition graph of $\ptask$. 
We say that $\tuple{\ptask,\statesmap}$ is an abstraction of $\mdptask$ 
if for all $\tuple{s,\action,t}$ we have $\transprob(t|s,\action)>0$, iff
$\tuple{\statesmap(s),\op,\statesmap(t)}\in T_{\ptask}$ for some $\op\in\ops$ or
$\statesmap(s)=\statesmap(t)$. 
We call such PaRL tasks proper.
\end{definition}

The idea behind the definition of PaRL task 
is to allow the specification of  some of the 
functionality of the reinforcement learning task in a declarative way. 
In what follows, we only consider proper PaRL tasks.
Next, we link the RL task $\mdptask$ and the planning task $\ptask$ by an options framework.
\begin{definition}
\label{def:plan_options}
For a PaRL task $E\!:=\!\tuple{\mdptask,\!\ptask,\!\statesmap}$,
\textbf{plan options} are:
(1) for each operator $\op\!\in\!\ops$ in $\ptask$,
an \textbf{operator option}
$\option_{\op}\!:=\!\tuple{\initset{\option_{\op}},\!\pi_{\option_{\op}},\!\termset{\option_{\op}}}$ 
with 
$\initset{\option_{\op}}\!:=\!\{\!\state\!\in\!\states\!\!\mid\!\!\pre(\op)\!\subseteq\!\statesmap(\state)\}$
and 
$\termset{\option_{\op}} := \{ \state \in \states \mid (\prv(\op) \cup \eff(\op)) \subseteq \statesmap(\state) \}$,
and
(2) 
a single \textbf{goal option}
$\goaloption:=\langle  \initset{\goaloption}, \pi_{\goaloption}, \termset{\goaloption}\rangle$ 
with 
$\initset{\goaloption} := \{ \state \in \states \mid \goal \subset \statesmap(\state)\}$ and 
$\termset{\goaloption}\!:=\!\Goal$.
\end{definition}

Previous attempts in the literature have suggested utilizing planning operators 
to define options \cite{lyu-et-al-aaai2019,illanes-et-al-icaps2020}. 
However, earlier works assume an 
additional domain knowledge associating planning operators with conditions over propositional variables.
Here, we do not require such additional input, relying solely on the planning task. 

Denoting by $\cO_{\mdptask}$,
a set of plan options induces an SMDP 
$\mdptask':=\tuple{\states, \cO_{\mdptask}, \transprob, \transreward, S_0, \Goal, \gamma}$,
where we replace the primitive actions $\actions$ in $\mdptask$ 
with  $\cO_{\mdptask}$.
Next, 
we define a transition graph $\cT_{\mdptask'}$ of  $\mdptask'$
in which a multi-step state transition of an option $\option_{\op}$ is collapsed to 
a single labeled transition that 
connects each 
state $s \in \initset{\option_{\op}}$ to the states $t \in \termset{\option_{\op}}$.
\begin{definition}
Given a PaRL task $E:=\tuple{\mdptask, \ptask, \statesmap}$,
a transition graph of the SMDP 
$\mdptask':=\tuple{\states, \cO_{\mdptask}, \transprob, \transreward, s_0, \Goal, \gamma}$
is a triple 
$\cT_{\mdptask'}:=\tuple{\states, T_{\mdptask'}, \Goal}$,
where $\states$ is the states of $\mdptask$,
$T_{\mdptask'}$ is a set of non-deterministic labeled transitions
$\{ \tuple{s, \op, t} \mid s \in \initset{\option_{\op}}, t \in \termset{\option_{\op}},
P(t | s, \option_{\op}) >0\}$,
and $\Goal$ is the goal states in $\mdptask$.
\end{definition}

\subsection{Frames and Decompositions in Plan Options}
Although we do not assume to have an exact model of $\mdptask$, 
it is desirable to have an annotating planning task $\ptask$
that behaves similar to $\mdptask$.
To characterize the similarity between the two tasks,
we introduce a context and frame of an option $\option_{\op}$ in
an RL state $\state$
to capture 
the subset of facts in the planning task
that prevail 
when applying a planning operator $\op$ 
to the planning state $\statesmap(\state)$,
namely $\statesmap(\state) \cap \statesmap(\state)\applied{\op}$.
\begin{definition}
For an operator $\op$ and its option $\option_{\op}$,
we define the \textbf{context of an operator option in state}
$\state\in\states$ by
$
\cC_{\option_{\op}}(\state):=\statesmap(\state)\setminus (\pre(\op) \cup \eff(\op))$.
The \textbf{frame of an operator option in state} $\state \in \states$
is $\mathcal{F}_{\option_{\op}}(\state):= \prv(\op) \cup \cC_{\option_{\op}}(\state)$.
A partial frame of an option in $\state$
$\mathcal{F}^p_{\option_{\op}}(\state)$
is a subset of $\mathcal{F}_{\option_{\op}}(\state)$.
\end{definition}

We say that a PaRL task $E$ with a set of plan options $\cO_{\mdptask}$ is frame preserving
if $\mathcal{F}_{\option_{\op}}(s) = \mathcal{F}_{\option_{\op}}(t)$
for every $\tuple{s, \op, t} \in T_{\mdptask'}$ and operator $\op \in \ops$.

\begin{theorem}
If a PaRL task $E$ with plan options $\options_{\mdptask}$
is frame preserving,
then $\cT_{\ptask}$ and $\cT_{\mdptask'}$ are bisimilar.
\end{theorem}
\begin{proof}
Consider a binary relation 
$\{\tuple{s, t}\!\!\in\!\!\states\!\times\!\states\!\mid\!\statesmap(s)\!\!=\!\!\statesmap(t)\}$.
For each $\op \in \ops$,
every $\tuple{s, \op, t} \in T_{\mdptask'}$
satisfies
$\statesmap(t)\!\!=\!\!\big[\statesmap(s)\setminus\!\big(\pre(\op)\!\cup\!\cF_{\option_{\op}}(\state)\big)\big]
\cup\big(\eff(\op)\!\cup\!\cF_{\option_{\op}}(\state)\big)\!=\!\statesmap(\state)\applied{\op}.$
For a transition $\tuple{t, \op, t'} \in T_{\mdptask'}$ such that  $\statesmap(t)=\statesmap(s)$,
$\statesmap(t')=\statesmap(t)\applied{\op}=\statesmap(s)\applied{\op}$.
\end{proof}

The desiderata in HRL is that 
a task hierarchy in $\mdptask$
captures the decomposition  
into sub-MDP tasks that are easier to solve in a local state space,
and those sub-tasks are reusable in similar problems.
It is often claimed that HRL improves sample efficiency,
and 
a sub-problem analysis by \citet{wen-et-al-nips2020}
shows
that HRL methods can improve the sample efficiency
if the total sum of 
the size of each partitioned state space is smaller than 
the size of the original state space.
Following this intuition behind the MDP decomposition in HRL,
we now characterize the sub-problem decomposition imposed by the frame-constrained option MDPs.


%
%
\begin{definition}
Given a PaRL task $E\!\!:=\!\!\tuple{\mdptask,\ptask,\statesmap}$ and a plan option 
$\option_{\op}\!\!:=\!\!\tuple{\initset{\option_{\op}},\!\pi_{\option_{\op}},\!\termset{\option_{\op}}}$,
a \textbf{frame constrained option MDP} is 
an MDP for a Markovian option $\option_{\op}$,
defined as
$$\mdptask_{\op, s_0}:=
\langle\states_{\cF_{\op}(s_0)}, \actions, 
\transprob_{\cF_{\op}(s_0)}, \transreward, s_0, \termset{\option_{\op}},
\cD_{\cF_{\op}(s_0)}, \gamma\rangle,$$
where
$\states_{\cF_{\op}(s_0)}$ is the local states,
$\transprob_{\cF_{\op}(s_0)}$ is a constrained state transition probability,
the initial state $s_0$ is a state in $\initset{\option_{\op}}$,
$\termset{\option_{\op}}$ are the goal states,
and 
$\cD_{\cF_{\op}(s_0)}$ is 
a set of fictitious transition constraints
$\{\cF_{\option_{\op}}(t)=\cF_{\option_{\op}}(s_0)\mid
\forall \tuple{s,\action,t}\in\cT_{\mdptask_{\op, s_0}}, \pi_{\option_{\op}}(\action|s)>0\}$,
enforcing the state transitions to preserve $\cF_{\op}(s_0)$.
\end{definition}
Note that 
we modified $\transprob_{\cF_{\op}(s_0)}$ 
in $\mdptask_{\op, s_0}$
from the original $\transprob$ in $\mdptask$
so that all the transitions don't violate $\cD_{\cF_{\op}(s_0)}$:
assign $\transprob_{\cF_{\op}(s_0)}(t|s, \action) = 0$ to
all $\tuple{s, \action, t} \in \cT_{\mdptask_{\op, s_0}}$ such that $\cF_{\op}(s_0) \not\subset \statesmap(t)$, and then normalize conditional probability.

Introducing the frame constraints to each option MDP
reduces the size of the state space subject to the number of facts in the frame of the option.
\begin{theorem}
Given a PaRL task $E:=\tuple{\mdptask,\ptask,\statesmap}$
and 
two frame-constrained option MDPs 
$\mdptask^p_{\op, s_0}$ and $\mdptask^q_{\op, s_0}$
induced by 
partial frames 
$\cF_{\option_{\op}}^{p}(s_0)$ and $\cF_{\option_{\op}}^{q}(s_0)$,
if
$\cF_{\option_{\op}^{p}}(s_0) \subset \cF_{\option_{\op}^{q}(s_0)}$,
the states of $\mdptask^q_{\op, s_0}$ are states of $\mdptask^p_{\op, s_0}$. 
\end{theorem}
\begin{proof}
Let $\states_{p}$ and $\states_{q}$
denote the states of the 
$\mdptask^p_{\op, s_0}$ and $\mdptask^q_{\op, s_0}$.
For every $\state \in \states_q$, 
we can see that $\state \in \states_p$
since 
$\cF_{\option_{\op}}^{p}(s_0) \subset \cF_{\option_{\op}}^{q}(s_0) \subset \statesmap(\state)$.
\end{proof}
If 
all option MDP are frame-constrained and the PaRL task is frame-preserving,
we may have two advantages:
(1) improved sample efficiency due to the reduction in the state space size for learning options, and
(2) options re-usability by composition, relying solely on the symbolic annotation.

\subsection{Intrinsic Rewards for Plan Options}
In practice, 
we don't assume 
that the annotating planning task $\ptask$ simulates 
the underlying $\mdptask$, and
furthermore, 
it is impossible to constrain the 
transitions in the MDP task in RL.
Therefore,
we relax all the constraints in the frame-constrained option MDPs and 
absorb those constraints in the objective function 
as an intrinsic reward to the option learning agent.

\begin{definition}
Given a PaRL task $E\!\!:=\!\!\tuple{\mdptask,\ptask,\statesmap}$ and a plan option 
$\option_{\op}\!\!:=\!\!\tuple{\initset{\option_{\op}},\!\pi_{\option_{\op}},\!\termset{\option_{\op}}}$,
a \textbf{frame penalized option MDP}
is a tuple 
$\overline{\mdptask}_{\op, s_0}:=\{\states, \actions, P, \overline{\transreward}, 
\overline{s}_0, \termset{\option_{\op}}, \gamma\}$,
where
we replace 
the reward function,
the initial state,
and the goal of the MDP task $\mdptask$
with 
an intrinsic reward $\overline{\transreward}$,
an initial state $\overline{s}_0 \in \initset{\option_{\op}}$,
and
$\termset{\option_{\op}}$, respectively.
Under the objective that maximizes the expected sum of discounted rewards,
the intrinsic reward $\overline{\transreward}:=\states\rightarrow\mathbb{R}$ is given by
$$
\overline{\transreward}(\state):=
\!\!\!\!\!\!\!\!\!\!\sum_{\var\in\vars\big(\cF_{\option_{\op}}(\overline{s}_0)\big)}
\!\!\!\!\!\!\!\!\!\!c_1 \cdot \mathbb{I}\big(\statesmap(\state)[\var]\neq\cF_{\option_{\op}}[\var]\big)
+
c_2 \cdot 
\mathbb{I}\big(\state \not\in \termset{\option_{\op}}\big)
,
$$
where 
$\mathbb{I}$ is an indicator function and $c_1$ and $c_2$ are negative rewards. 
\end{definition}
Note that the state space of the frame penalized option MDP 
can be as large as the original state space.
We only hope that
the intrinsic reward obtained in the planning space
guides the option policy learning agent to visit states 
that are more likely to preserve the frame of an option.
%
%
In the absence of knowledge about the underlying dynamics of $\mdptask$, 
an SMDP task $\mdptask'$ induced by the plan options 
also solves $\mdptask$ yet with a lower expected return
if $\ptask$ does not have a dead-end.
Namely, 
if $\mdptask$ reaches the goal in discounted stochastic shortest path model \cite{bertsekas2018abstract}, 
$\mdptask'$ will reach the goal with a finite yet larger number of steps. 
If $\mdptask$ has dead-ends and maximizes the probability of reaching the goal \cite{kolobov-et-al-uai2012},
$\mdptask'$ will also reach the goal with a lower yet non-zero probability.

\begin{algorithm}[t]
\caption{Online Option Learning with a PaRL Task}
\label{alg:online-parl}
\begin{algorithmic}[1]
\Require PaRL task $E\tuple{\mdptask, \ptask, \statesmap}$. 
\Ensure Option policies $\pi_{\option_{\op}}(\action|\state)$.
\State Initialize trajectory buffer $B$
\State Initialize a set $D$ for storing options 
\While {$iter<N$}
    \Statex{\textbf{\phantom{xx}rollout samples from the current option policies}}
    \While {$iter_{\text{rollout}}<N_{\text{rollout}}$}
        \State $s \leftarrow$ current state
        \State Select an option $\option_{\op}$ by AI planner
        \If {$\option_{\op} \not\in D$}
            \State Create $\option_{\op}$, Initialize $\pi_{\option_{\op}}$, 
            Add $\option_{\op}$ to $D$
        \EndIf
        \While {$s \not\in \termset{\option_{\op}}$}
            \State Sample $(s, a, r_e, t)$ using $\pi_{\option_{\op}}$
            \State Compute intrinsic reward $r_i$ 
            \State Store $(\option_{\op}, s, a, r_e, r_i)$ to buffer $B$
            \State $s \leftarrow t$
        \EndWhile
    \EndWhile
    \Statex{\phantom{xx}\textbf{train policies}}
    \For {each option $\option_{\op} \in D$}
        \State Train option policy function $\pi_{\option_{\op}}$ with RL
    \EndFor
\EndWhile
\end{algorithmic}
\end{algorithm}

\begin{algorithm}[t]
\caption{HplanPPO: Online Option Learning with PPO}
\label{alg:smdp}
\begin{algorithmic}[1]
\Require PaRL task $E\tuple{\mdptask, \ptask, \statesmap}$. 
\Ensure option policies $\pi_{\option_{\op}}(\action|\state)$
\State Initialize trajectory buffer $B$
\State Initialize a dictionary $D[L(s)\!\!:\!\!\{\option_{\op}\}]$ for storing options 
\While {$iter<N$}
    \State $s \leftarrow$ current state
    \State Initialize a set $E$ for storing the unrolled options
    \While {$iter_{\text{rollout}}<N_{\text{rollout}}$}
        \State $\tuple{\op_1, \op_2, \ldots, \op_k} \leftarrow$ Planner($L(s)$, $\goal$)
        \State $t' \leftarrow L(s)$
        \For{each option $\op$ in $\tuple{\op_1, \op_2, \ldots, \op_k}$}
            \If{$\op \not\in D.keys()$}
                \State Initialize $\pi_{O_{\op}}$ and Add $\op$ to $D[t']$    
                \State $t' \leftarrow t'\applied{\op}$
            \EndIf
        \EndFor
        \State Select $\option_{\op}$ from $D[L(s)]$ and Add $\option_{\op}$ to $E$
        \While{$s \not\in \termset{\option_{\op}}$}
            \State Generate on-policy sample $(s, a, r_e, t)$
            \State Compute intrinsic reward $r_i$ 
            \State Store $(\option_{\op}, s, a, r_e, r_i)$ to buffer $B$
            \State $s \leftarrow t$
        \EndWhile
    \EndWhile
    \For {each option $\option_{\op} \in E$}
        \State Train $\pi_{\option_{\op}}$ by PPO
    \EndFor
\EndWhile
\end{algorithmic}
\end{algorithm}

\begin{figure*}
\centering
\begin{subfigure}[b]{0.24\textwidth}
  \includegraphics[width=\textwidth]{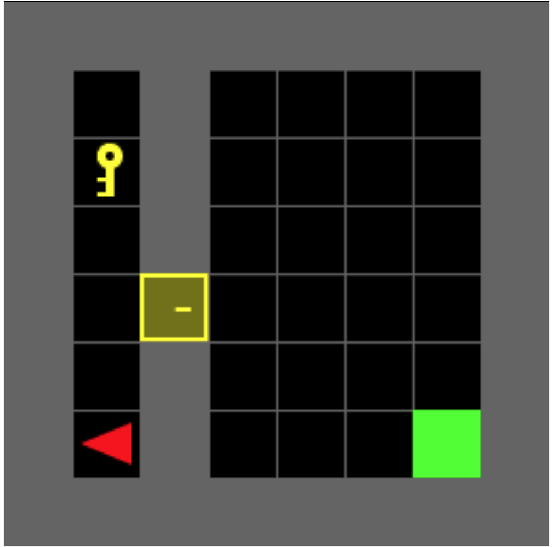}
  \caption{MiniGrid Door Key}
  \label{f11}
\end{subfigure}
\begin{subfigure}[b]{0.24\textwidth}
  \includegraphics[width=\textwidth]{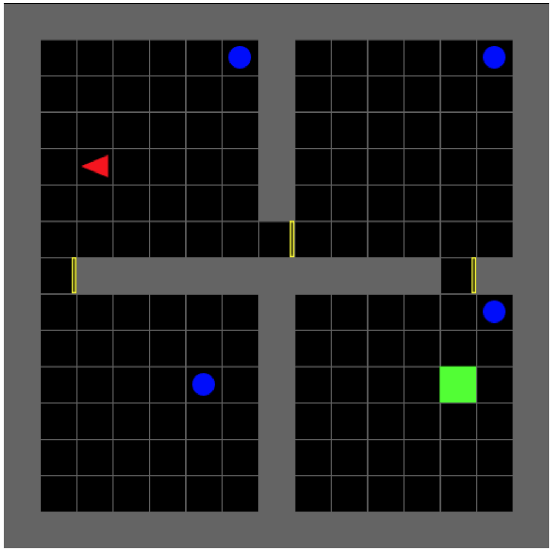}
  \caption{4 Rooms with Balls}
  \label{f12}  
\end{subfigure}
\begin{subfigure}[b]{0.24\textwidth}
  \includegraphics[width=\textwidth]{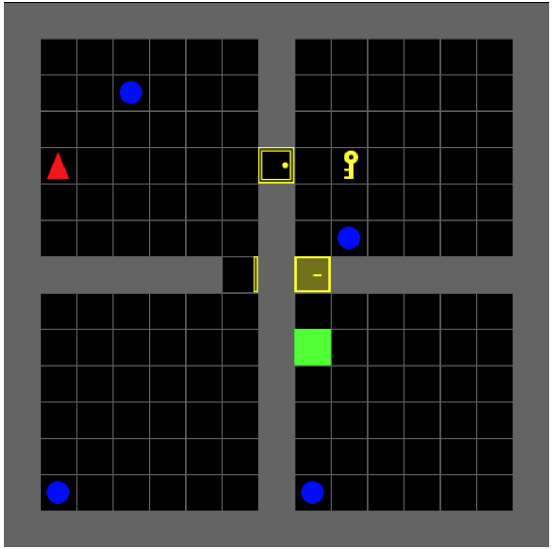}
  \caption{4 Rooms with a Locked Door}
  \label{f13}
\end{subfigure}
\begin{subfigure}[b]{0.24\textwidth}
  \includegraphics[width=\textwidth]{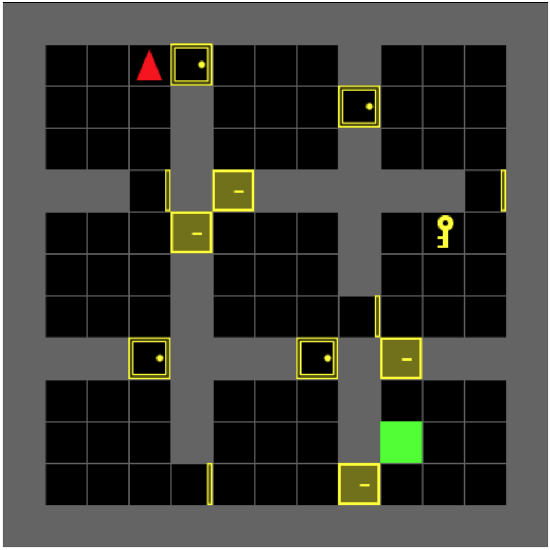}
  \caption{9 Rooms with Locked Doors}
  \label{f14}
\end{subfigure}
\caption{Example of \texttt{MiniGrid}-based instances:
From left to right, we show one instance from each RL problem domain
with additional task structures. 
A single planning task annotates each problem domain, and 
each domain randomizes the location of the objects. 
}
\label{f1}
\end{figure*}

\section{Solving PaRL Task}
In this section,
we present HRL algorithms for solving PaRL tasks $E:=\tuple{\mdptask,\ptask,\statesmap}$.
%
For any pair of initial state $s_0 \in \states$ and a goal $s_g \in \Goal$ in $\mdptask$,
we can generate a sequence of options \{$\option_{o_1}, \option_{o_2}, \ldots, \option_{o_k}$\}
from a plan in $\ptask$ that reaches the goal state $L(s_g) \in \states'$ from the initial state $L(s_0) \in \states'$.
Therefore, we can invoke AI planners in two ways,
either precompute those option-level plans offline or 
generate plans while training option policies online.
In this paper, we focus on the online approach
that integrates AI planner as a higher level control agent
and model-free RL as lower-level agents in HRL architecture.

Algorithm \ref{alg:online-parl}
shows the outline of HRL agent 
that learns options online with a PaRL task. 
The algorithm alternates rollout and training phases until the limit on the number of iterations is reached.
In the rollout phase, 
HRL agent selects an option $\option_{\op}$ using any AI planner
by solving the annotating planning task $\Pi$ with the initial planning state
$s'_{0}$ being the current planning state $L(s)$
and returning the applicable planning operator in $s'{0}$ (line 6).
Note that this re-planning at every option selection step
formulate the option selection 
as an action selection in online planning \cite{mausam-kolobov-2012},
it is much more efficient than learning-based approaches.
If $\option_{\op}$ was not created before, 
we create the option and initialize the policy $\pi_{\option_{\op}}$ and add it to a container
(lines 7-8).
Next, sample trajectories are generated by using $\pi_{\option_{\op}}$ until it terminates,
and 
for each one-step state transition, we compute the intrinsic reward following Definition 7
(lines 10-12).
Then, HRL agent updates the option policy 
using the samples stored in the buffer in the training phase
by any model-RL algorithm (line 15).

In our experiments, 
we integrated $A^{*}$ algorithm implemented in \texttt{Pyperplan} \cite{alkhazraji2016pyperplan}
with double DQN (\texttt{DDQN}) \cite{van2016deep} and Proximal Policy Optimization \texttt{PPO} \cite{PPO}
for option policy training, yielding two algorithms: 
HRL with integrated planning and PPO (\texttt{HplanPPO}) and HRL with integrated planning and DDQN
\texttt{HplanDDQN}.
Algorithm \ref{alg:smdp} shows \texttt{HplanPPO},
creating a separate \texttt{PPO} agent per option. Option training 
phases do not share samples (an on-policy method).
\texttt{HplanDDQN} is similar to \texttt{HplanPPO} except 
for it only creates a single \texttt{DDQN} agent 
that augments the input to \texttt{DDQN} network with a one-hot encoding of option labels, allowing 
reuse of the samples for training option policies.


\section{Experiments}
All experiments are conducted in a cluster computing environment
equipped with Intel (R) Xeon(R) Gold 6258R CPUs 
and NVIDIA A-100/V-100 GPUs. 
For each run, we limited computational resources to utilize up to 16 GB of memory with 2 CPUs and 1 GPU.
For HRL experiments, 
we created two benchmark sets.
The first benchmark set extends 
\texttt{MiniGrid} environment \cite{gym_minigrid}
with more complex task structures than
the predefined environments
inspired by \texttt{BabyAI} environment \cite{babyai_iclr19}.
During training RL/HRL algorithms, 
we reset the environments with the same random seeds from 0 to 999 and 
excluded existing baseline algorithms relying on tabular Q-learning.

The second benchmark set extends
four rooms navigation domain on a larger scale with a maze-like topology.
Namely, we extend to 4 rooms on a 20x20 grid and
12 rooms on a 16x16 grid only allowing
a single path that connects all the rooms.

\paragraph{Algorithms}
We evaluate \texttt{HplanPPO} and \texttt{HplanDDQN} algorithms
and compare them with existing baseline algorithms
from the flat RL counterparts, \texttt{PPO} and \texttt{DDQN},
and Deep Hierarchical Reward Machines  (\texttt{HRM}) \cite{icarte2022reward}.
We implemented all algorithms 
by extending \texttt{stable-baselines3} \cite{stable-baselines3},
except for \texttt{HRM}, which offers open-source implementation 
by the original authors. 
In the experiments in \texttt{MiniGrid} environments, 
we evaluated each algorithm at least 5 times,
and report the average, the minimum and maximum range, and 95 percent
confidence intervals in the plots.
Due to the space limit, we will provide
details 
on general experiment setups and
implementation of 
deep neural network architectures and 
hyper-parameter choices in the Appendix.

\begin{figure*}[t]
\centering
    \begin{subfigure}[b]{0.24\textwidth}
    \includegraphics[width=\textwidth]{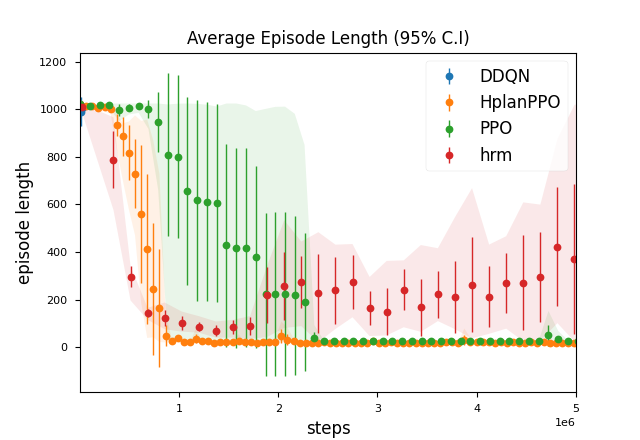}
      \caption{DoorKey}
      \label{f21}  
    \end{subfigure}
    \begin{subfigure}[b]{0.24\textwidth}
      \includegraphics[width=\textwidth]{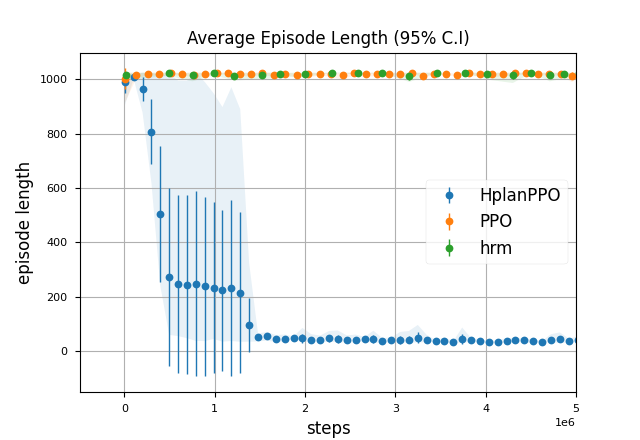}
      \caption{4 Rooms with Balls}
    \label{f22}  
    \end{subfigure}
    \begin{subfigure}[b]{0.24\textwidth}
      \includegraphics[width=\textwidth]{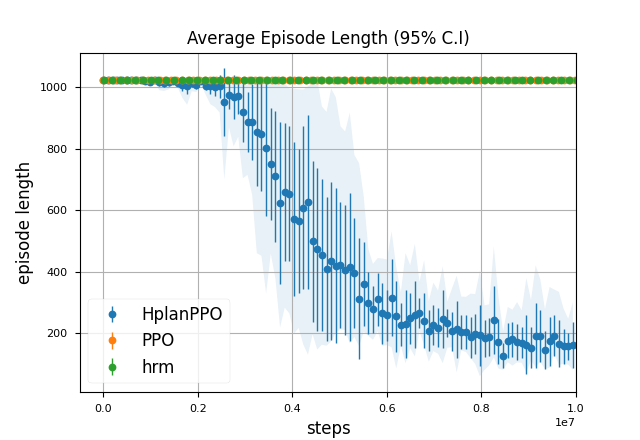}
      \caption{4 Rooms with a Locked Door}
      \label{f23}
    \end{subfigure}
    \begin{subfigure}[b]{0.24\textwidth}
      \includegraphics[width=\textwidth]{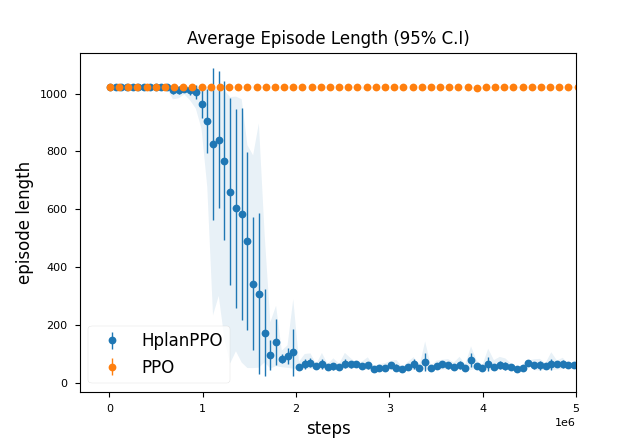}
      \caption{9 Rooms with Locked Doors}
      \label{f24}
    \end{subfigure}
\begin{subfigure}[b]{0.24\textwidth}
\includegraphics[width=\textwidth]{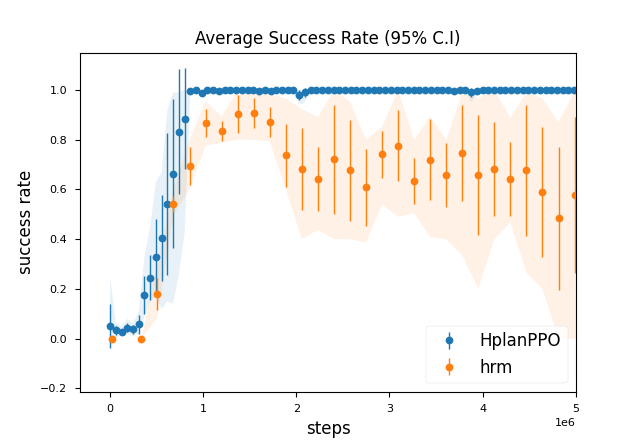}
  \caption{DoorKey}
  \label{f25}  
\end{subfigure}
\begin{subfigure}[b]{0.24\textwidth}
  \includegraphics[width=\textwidth]{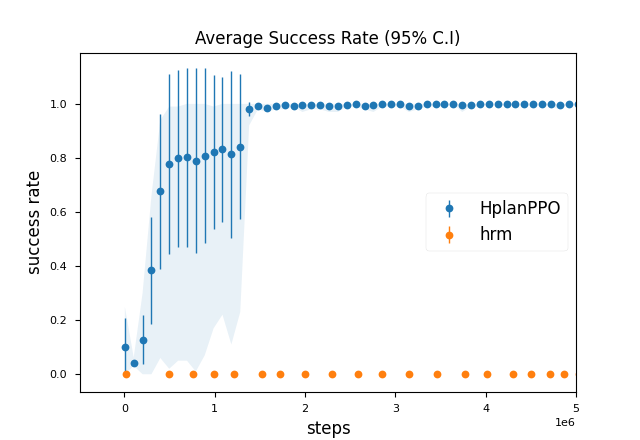}
  \caption{4 Rooms with Balls}
  \label{f26}  
\end{subfigure}
\begin{subfigure}[b]{0.24\textwidth}
  \includegraphics[width=\textwidth]{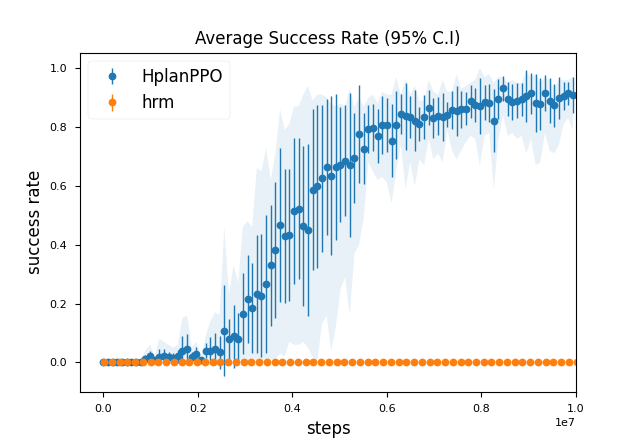}
  \caption{4 Rooms with a Locked Door}
  \label{f27}
\end{subfigure}
\begin{subfigure}[b]{0.24\textwidth}
  \includegraphics[width=\textwidth]{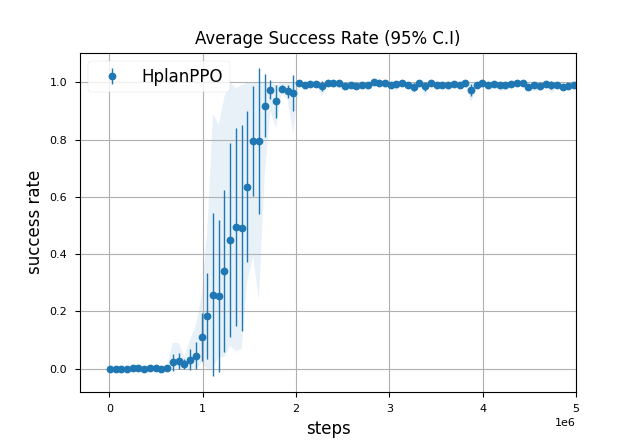}
  \caption{9 Rooms with Locked Doors}
  \label{f28}
\end{subfigure}
\caption{Average episode length (a)-(d) and average success rate (e)-(h) on \texttt{MiniGrid}-based instances. 
}
\label{f2}
\end{figure*}

\subsection{MiniGrid-based Benchmark Problems}
Figure \ref{f1} illustrates 
examples of four \texttt{MiniGrid}-based RL environments.
Each RL environment 
can vary the location of objects such as a key, door, or the goal tile.
In addition to such random variations, 
we introduced the following task structures that can be 
captured by a single planning annotation task per RL domain.
\texttt{MiniGrid Door Key} in Figure \ref{f11},
the high-level task of the agent is to pick up the key, unlock the door
and move to the goal location.
\texttt{Four rooms with Balls} and \texttt{Four rooms with a locked door}
both share a similar structure except for the 
\texttt{Four rooms with a locked door} domain requires
an agent to use the key to unlock the door.
The last \texttt{Nine rooms with locked doors} increase
the number of rooms.
Note that the symbolic state space in the planning annotation task
abstracts away several details in MDP. 
First, the agent doesn't know the precise location in the grid,
but it knows the room that the agent stays in. 
Second, balls are invisible to the planning task, and the door state is also partially known to the planning agent. Namely, the planning agent cannot detect whether the door is closed or not.
In addition to these modifications,
we also assume that the environment is \textbf{fully-observable} MDP.

\subsection{Comparison against Baselines}
Figure \ref{f2} shows 
the average episode length and the success rate 
for solving \texttt{MiniGrid} based problem domains.
Figure \ref{f21} and \ref{f25}
show the result from \texttt{MiniGrid DoorKey}.
First, we can see that both \texttt{HplanPPO}
and \texttt{HRM} indeed improved the sample efficiency
compared with flat RL baselines such as \texttt{PPO}.
Comparing \texttt{HRM} and \texttt{HplanDDQN},
we see that \texttt{HplanDDQN} does not show notable 
progress in learning. 
We also observed that flat \texttt{DDQN} showed similar trends.
Lastly, Figure \ref{f21} shows that
the performance of \texttt{HplanPPO} is more stable
than \texttt{HRM} since 
\texttt{HplanPPO} does not need to train the high-level control policy,
solving instead the planning annotation task using AI planner.

In the four rooms domain and nine rooms domain,
all other baseline algorithms, including \texttt{HRM}
did not show notable learning progress.
\texttt{HplanPPO} was the only algorithm
that showed consistent and stable learning performance.
First, the dimension of observation rapidly increase
in 4 rooms domain and 9 rooms domain, 
each generating three channels of 15x15 and 22x22 arrays.
Second, the underlying MDP of \texttt{MiniGrid} environment
only feedbacks the agent with a sparse reward,
$(1 - \frac{0.9}{l})$ with $l$ being the length of the episode.
In \texttt{HplanPPO}, 
it utilizes a more informative reward signal
due to the intrinsic rewards derived from
the planning states and the shorter episode length
per terminating options.


\begin{figure*}
\centering
\begin{subfigure}[b]{0.33\textwidth}    
\includegraphics[width=\textwidth,height=4.5cm]{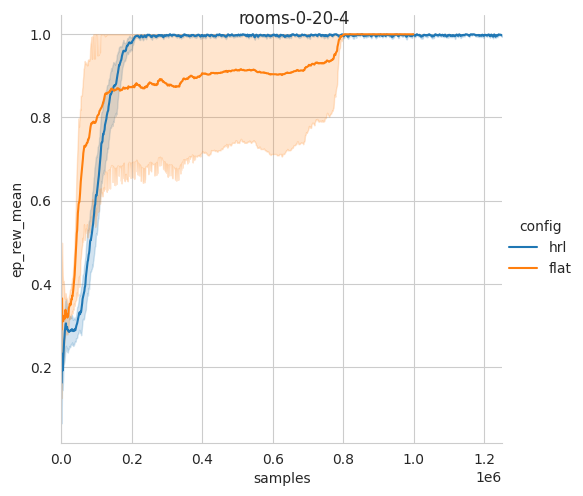}
\caption{Success rate on 20x20 grid}
\label{f31}
\end{subfigure}
\begin{subfigure}[b]{0.33\textwidth}    
\includegraphics[width=\textwidth]{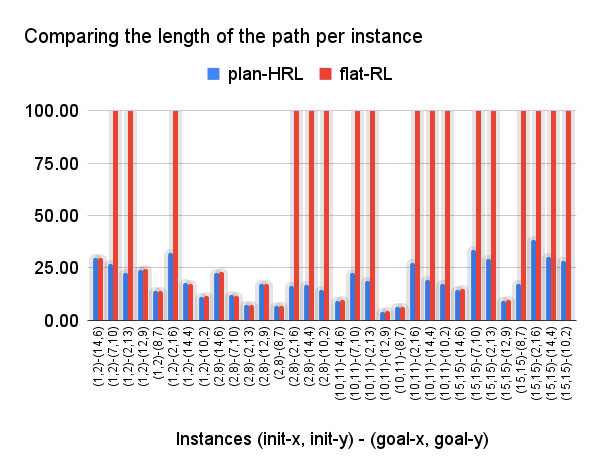}
\caption{
Average episode length}
\label{f32}
\end{subfigure}
\begin{subfigure}[b]{0.33\textwidth}    
\includegraphics[width=\textwidth]{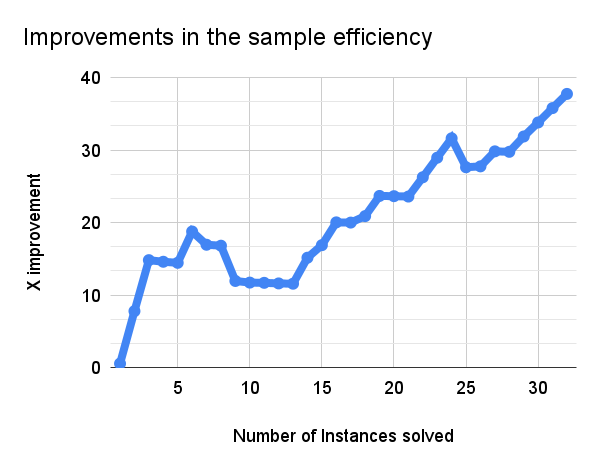}
\caption{
Sample efficiency improvement rate}
\label{f33}
\end{subfigure}   
\caption{Success rate, average length, and improvement in sample efficiency on various size instances of rooms domain. (a) 4 rooms, 20x20 grid, (b)-(c) 12 rooms, 16x16 grid.}
\label{f3}
\end{figure*}

\subsection{Rooms Domain}
\texttt{N-rooms} domain modifies the classic \texttt{4-rooms} domain by
increasing the number of rooms and varying the size of the grids.
In this problem, an agent moves on a grid, 
separated into N rooms, with narrow corridors connecting the adjacent rooms.
The agent needs to move from a given location on the grid (in a given room) to a goal location (in a goal room). 
 
In order to use the domain within our framework, 
we have created an abstract version of \texttt{N-rooms} planning domain
and defined the state mapping function $L$ from
the feature vectors in $\mdptask$ encoding the coordinate on the grid to
boolean vectors of the ground propositions in $\ptask_M$\footnote{Both the domain file and the problem files can be found in Appendix.}.
%
%
As an example, in a four rooms domain on a N$\times$N grid, 
the PaRL task captures a total of 16 options
that define all the movements between rooms and their adjacent corridors. 

%

Figure \ref{f31} shows the learning curve 
from a 4 rooms domain on a 20x20 grid, comparing 
the success ratio for reaching the goal.
In both \texttt{PPO} and \texttt{HplanPPO},
we gave a reward of +1 when reaching the goal,
and a cost of -0.05 for each step.
Comparing the results,
we can see that 
the learning curves  from \texttt{HplanPPO} (blue)
reach a higher reward much earlier than \texttt{PPO} (orange).

Figure \ref{f32} compares the length of the path 
from an initial state to a goal state in a 16x16 grid 12 rooms domain.
The x-axis shows different combinations of initial states and the goal state
and the y-axis shows the number of steps to reaching the goal.
In both \texttt{PPO} and \texttt{HplanPPO} cases, we allowed RL agents to reuse the policies trained in the previous problems.
Figure \ref{f33} shows the improvement of the sample efficiency
in \texttt{HplanPPO} when the options are reused over 32 different problem instances.

\section{Related Work}
Some of the relevant work that attempt to combine symbolic planning and RL include 
PEORL \cite{yang-et-al-ijcai2018}, SDRL \cite{lyu-et-al-aaai2019}, 
and Taskable RL \cite{illanes-et-al-icaps2020}. 
In Taskable RL, a manual mapping between 
the high-level actions in the planning task
and the options in the RL task must be provided,
and such a mapping could be many to one.
The termination set also requires manual modifications.
This makes it difficult to apply the method to
a new RL task or even a different problem instance in the same problem domain.
%
%
%
While PEORL also considers integrating symbolic planning and RL, 
it assumes that an exact representation of MDP is
available in the planning task and that there is one option per planning transition, 
which is unrealistic in many domains. 
It updates
the value function and the associated option policies only when options terminate, while
our method can operate online to accommodate intra-option updates. 
Moreover, PEORL is based on tabular representation, while our method is based on deep
neural network representation. 
SDRL, the deep learning extension of PEORL, still
learns both the lower-level and the high-level policies in restricted setting, 
whereas our method can use any RL algorithm for option policy learning.
A complementary to our work is the work is model learning \cite{jin2022creativity}, 
discovering the planning model that we assume to be given. 

Recently, there has been a series of work on using reward machines and linear temporal logic (LTL) for defining finite state automata encapsulating high-level symbolic information. The information is used to define either reward \cite{camacho2019ltl,icarte2022reward} or options \cite{icarte2018using}. The main difference of these methods from ours is that these methods encode abstract or partial {\em solutions}, and thus require someone to first solve the problem, at least on an abstract level. These solutions need to be updated once moving to solving even somewhat similar tasks. Further, while LTLs allow for capturing complex logical expressions, only a very restricted fragment is within reach of RL algorithms, even theoretically \cite{yangtractability}.

\section{Conclusions and Future Work}
In this work, we have presented a simple general framework for annotating
reinforcement learning tasks with planning tasks, to facilitate the transfer of
planning based techniques into the field of reinforcement learning.
Our framework links the state abstraction in AI planning and temporal abstraction in RL, providing a way to decompose the MDP into sub-MDPs by specifying options initiation set and termination condition based on planning operator definitions. 
We design a general method for injecting intrinsic rewards to RL agents from the abstract planning task by reformulating the underlying decomposed sub-MDPs with constraints visible to planning agents.
Learning only the (intra-option) policies for these sub-MDPs is shown to work well in practice on various problems, significantly improving sample efficiency.

This, however, is not the end of the road. While this work focused on temporal abstractions, our framework is more general, allowing to inject knowledge from the annotated planning task into the MDP. 
One example of such knowledge is goal distance estimates that can be used for reward shaping \cite{GehringPRL}. Another example is landmarks \cite{porteous-et-al-ecp2001} (logical formula that must occur on all plans), that can be used as sub-goals. Planning research has been focusing for years on automatically extracting knowledge from the planning task description. We believe that injecting this knowledge into the MDP can greatly improve the performance of RL agents.


\onecolumn
\clearpage

\appendix
\setcounter{secnumdepth}{2}

\section{Planning Annotations}
This section summarizes the planning annotations for
\texttt{MiniGrid} and \texttt{N Rooms} domains
that we evaluated in the experiment section.

\subsection{\texttt{MiniGrid}}
\label{app:mgdomain}
The RL environment maintains 
rooms over an NxN grid, 
blue balls, 
a green goal tile, 
the agent location and orientation, and
doors with states, open, closed, locked, and unlocked.
In planning tasks,
we abstract away information relevant to each cell 
in the grid.
Namely, 
the exact location and orientation of the agent,
the exact location of the key,
blue balls, and a green goal tile
are all ignored.
In addition, 
the states of a door
is simplified to two states, locked or unlocked.

On resetting the RL environment,
we implemented gym environments 
such that
objects that are only visible to RL environments are randomized,
as usual in the standard \texttt{MiniGrid} gym environment.
However, we restricted the information 
relevant to the planning task remains the same.
For example, the agent's initial location will be randomized 
within a predefined room (the room at the upper left corner),
and the goal location will also be randomized within a room 
at the lower right corner.
A key will appear in the same room,
and the initial state of the door, whether it is locked or unlocked,
will remain the same.
This choice doesn't limit algorithms but 
it simplifies the experiment to start with 
a single PDDL instance to annotate the environment,
although the agent will generate additional PDDL instances
when it solves planning tasks 
with a new initial planning state
while selecting options online.
\subsubsection{PDDL domain}
PDDL domain file
was manually generated by modifying existing similar PDDL domains.
\begin{verbatim}
(define (domain MazeRooms)
  (:requirements :strips :typing)
    (:types
        room - object
        key - object
        door - object
    )
    (:predicates
        (at-agent ?r - room)                        ; Agent current location 
        (at ?k - key ?r - room)                     ; Key location
        (carry ?k - key)                            ; Does agent carry the key
        (empty-hand)                                ; Is agent hand empty
        (unlocked ?d - door)                        ; Is door unlocked
        (locked ?d - door)                          ; Is door locked (for STRIPS only)
        (KEYMATCH ?k - key ?d - door)               ; Does the key match the door
        (LINK ?d - door ?r1 - room ?r2 - room)      ; Two rooms linked via the door
    )

    (:action move-room
        :parameters (?d - door ?r1 - room ?r2 - room)
        :precondition (and
            (LINK ?d ?r1 ?r2)
            (at-agent ?r1)
            (unlocked ?d)
        )
        :effect (and
            (not (at-agent ?r1))
            (at-agent ?r2)
        )
    )

    (:action pickup
        :parameters (?k - key ?r - room)
        :precondition (and
            (at ?k ?r)
            (at-agent ?r)
            (empty-hand)
        )
        :effect (and
            (not (at ?k ?r))
            (not (empty-hand))
            (carry ?k)
        )
    )

    (:action drop
        :parameters (?k - key ?r - room)
        :precondition (and
            (carry ?k)
            (at-agent ?r)
        )
        :effect (and
            (at ?k ?r)
            (empty-hand)
            (not (carry ?k))
        )
    )

    (:action unlock
        :parameters (?k - key ?d - door ?r1 - room ?r2 - room)
        :precondition (and
            (LINK ?d ?r1 ?r2)
            (KEYMATCH ?k ?d)
            (at-agent ?r1)
            (carry ?k)
            (locked ?d)
        )
        :effect (and
            (not (locked ?d))
            (unlocked ?d)
        )
    )

    (:action lock
        :parameters (?k - key ?d - door ?r1 - room ?r2 - room)
        :precondition (and
            (LINK ?d ?r1 ?r2)
            (KEYMATCH ?k ?d)
            (at-agent ?r1)
            (carry ?k)
            (unlocked ?d)
        )
        :effect (and
            (locked ?d)
            (not (unlocked ?d))
        )
    )
)
\end{verbatim}

\subsubsection{PDDL instance}
All PDDL problem instances were generated by our benchmark script code
by processing internal state information available in \texttt{MiniGrid}
gym environments.
\begin{figure}[h!]
    \centering
    \includegraphics[width=0.25\textwidth]{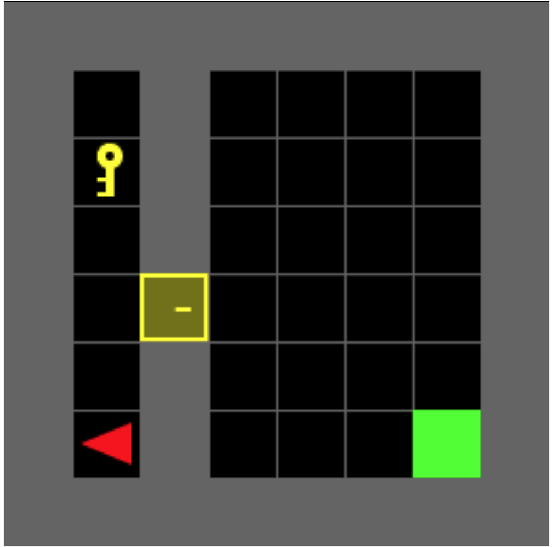}
    \caption{DoorKey}
    \label{fig:app_fig1}
\end{figure}
\begin{verbatim}
(define (problem MazeRooms-8by8-DoorKey)
    (:domain MazeRooms)
    (:objects
        R-0-0 R-1-0 -  room
        K-yellow-0 - key
        D-yellow-0-0-1-0 - door
    )
    (:init
        (LINK D-yellow-0-0-1-0 R-0-0 R-1-0)
        (LINK D-yellow-0-0-1-0 R-1-0 R-0-0)
        (KEYMATCH K-yellow-0 D-yellow-0-0-1-0)
        (at-agent R-0-0)
        (at K-yellow-0 R-0-0)
        (locked D-yellow-0-0-1-0)
        (empty-hand)
    )
    (:goal (and
        (at-agent R-1-0))
    )
)
\end{verbatim}

\begin{figure}[h!]
    \centering
    \includegraphics[width=0.25\textwidth]{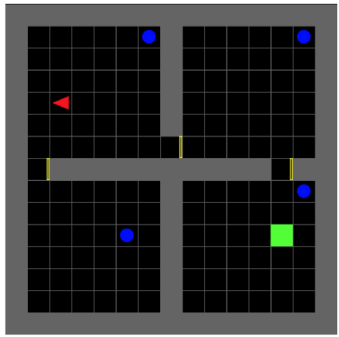}
    \caption{4 Rooms with Balls}
    \label{fig:app_fig2}
\end{figure}

\begin{verbatim}
(define (problem MazeRooms-2by2-Balls)
    (:domain MazeRooms)
    (:objects
        R-0-0 R-0-1 R-1-0 R-1-1 -  room
        D-yellow-0-0-1-0 D-yellow-1-0-1-1 D-yellow-0-0-0-1 - door
    )
    (:init
        (LINK D-yellow-0-0-0-1 R-0-0 R-0-1)
        (LINK D-yellow-0-0-0-1 R-0-1 R-0-0)
        (LINK D-yellow-0-0-1-0 R-0-0 R-1-0)
        (LINK D-yellow-0-0-1-0 R-1-0 R-0-0)
        (LINK D-yellow-1-0-1-1 R-1-0 R-1-1)
        (LINK D-yellow-1-0-1-1 R-1-1 R-1-0)
        (at-agent R-0-0)
        (unlocked D-yellow-0-0-0-1)
        (unlocked D-yellow-0-0-1-0)
        (unlocked D-yellow-1-0-1-1)
        (empty-hand)
    )
    (:goal (and
        (at-agent R-1-1))
    )
)
\end{verbatim}

\begin{figure}[h!]
    \centering
    \includegraphics[width=0.25\textwidth]{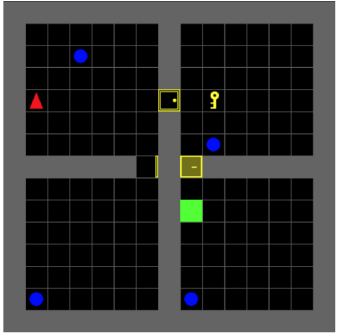}
    \caption{4 Rooms with a Locked Door}
    \label{fig:app_fig3}
\end{figure}

\begin{verbatim} 
(define (problem MazeRooms-2by2-Locked)
    (:domain MazeRooms)
    (:objects
        R-0-0 R-0-1 R-1-0 R-1-1 -  room
        K-yellow-0 - key
        D-yellow-0-0-1-0 D-yellow-1-0-1-1 D-yellow-0-0-0-1 - door
    )
    (:init
        (LINK D-yellow-0-0-0-1 R-0-0 R-0-1)
        (LINK D-yellow-0-0-0-1 R-0-1 R-0-0)
        (LINK D-yellow-0-0-1-0 R-0-0 R-1-0)
        (LINK D-yellow-0-0-1-0 R-1-0 R-0-0)
        (LINK D-yellow-1-0-1-1 R-1-0 R-1-1)
        (LINK D-yellow-1-0-1-1 R-1-1 R-1-0)
        (KEYMATCH K-yellow-0 D-yellow-0-0-0-1)
        (KEYMATCH K-yellow-0 D-yellow-0-0-1-0)
        (KEYMATCH K-yellow-0 D-yellow-1-0-1-1)
        (at-agent R-0-0)
        (at K-yellow-0 R-1-0)
        (locked D-yellow-1-0-1-1)
        (unlocked D-yellow-0-0-0-1)
        (unlocked D-yellow-0-0-1-0)
        (empty-hand)
    )
    (:goal (and
        (at-agent R-1-1))
    )
)
\end{verbatim}

\begin{figure}[h!]
    \centering
    \includegraphics[width=0.25\textwidth]{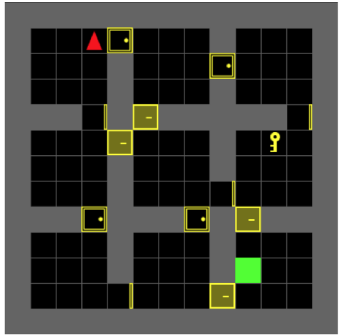}
    \caption{9 Rooms with Locked Doors}
    \label{fig:app_fig4}
\end{figure}

\begin{verbatim} 
(define (problem MazeRooms-3by3-LockedDoors)
    (:domain MazeRooms)
    (:objects
        R-0-0 R-0-1 R-0-2 R-1-0 R-1-1 R-1-2 R-2-0 R-2-1 R-2-2 -  room
        K-yellow-0 - key
        D-yellow-0-1-0-2 D-yellow-1-1-2-1 D-yellow-2-0-2-1 
        D-yellow-2-1-2-2 D-yellow-0-2-1-2 D-yellow-1-0-2-0 
        D-yellow-1-1-1-2 D-yellow-0-0-1-0 D-yellow-0-0-0-1 
        D-yellow-1-0-1-1 D-yellow-0-1-1-1 D-yellow-1-2-2-2 - door
    )
    (:init
        (LINK D-yellow-0-0-0-1 R-0-0 R-0-1)
        (LINK D-yellow-0-0-0-1 R-0-1 R-0-0)
        (LINK D-yellow-0-0-1-0 R-0-0 R-1-0)
        (LINK D-yellow-0-0-1-0 R-1-0 R-0-0)
        (LINK D-yellow-0-1-0-2 R-0-1 R-0-2)
        (LINK D-yellow-0-1-0-2 R-0-2 R-0-1)
        (LINK D-yellow-0-1-1-1 R-0-1 R-1-1)
        (LINK D-yellow-0-1-1-1 R-1-1 R-0-1)
        (LINK D-yellow-0-2-1-2 R-0-2 R-1-2)
        (LINK D-yellow-0-2-1-2 R-1-2 R-0-2)
        (LINK D-yellow-1-0-1-1 R-1-0 R-1-1)
        (LINK D-yellow-1-0-1-1 R-1-1 R-1-0)
        (LINK D-yellow-1-0-2-0 R-1-0 R-2-0)
        (LINK D-yellow-1-0-2-0 R-2-0 R-1-0)
        (LINK D-yellow-1-1-1-2 R-1-1 R-1-2)
        (LINK D-yellow-1-1-1-2 R-1-2 R-1-1)
        (LINK D-yellow-1-1-2-1 R-1-1 R-2-1)
        (LINK D-yellow-1-1-2-1 R-2-1 R-1-1)
        (LINK D-yellow-1-2-2-2 R-1-2 R-2-2)
        (LINK D-yellow-1-2-2-2 R-2-2 R-1-2)
        (LINK D-yellow-2-0-2-1 R-2-0 R-2-1)
        (LINK D-yellow-2-0-2-1 R-2-1 R-2-0)
        (LINK D-yellow-2-1-2-2 R-2-1 R-2-2)
        (LINK D-yellow-2-1-2-2 R-2-2 R-2-1)
        (KEYMATCH K-yellow-0 D-yellow-0-0-0-1)
        (KEYMATCH K-yellow-0 D-yellow-0-0-1-0)
        (KEYMATCH K-yellow-0 D-yellow-0-1-0-2)
        (KEYMATCH K-yellow-0 D-yellow-0-1-1-1)
        (KEYMATCH K-yellow-0 D-yellow-0-2-1-2)
        (KEYMATCH K-yellow-0 D-yellow-1-0-1-1)
        (KEYMATCH K-yellow-0 D-yellow-1-0-2-0)
        (KEYMATCH K-yellow-0 D-yellow-1-1-1-2)
        (KEYMATCH K-yellow-0 D-yellow-1-1-2-1)
        (KEYMATCH K-yellow-0 D-yellow-1-2-2-2)
        (KEYMATCH K-yellow-0 D-yellow-2-0-2-1)
        (KEYMATCH K-yellow-0 D-yellow-2-1-2-2)
        (at-agent R-0-0)
        (at K-yellow-0 R-2-1)
        (locked D-yellow-0-1-1-1)
        (locked D-yellow-1-0-1-1)
        (locked D-yellow-1-2-2-2)
        (locked D-yellow-2-1-2-2)
        (unlocked D-yellow-0-0-0-1)
        (unlocked D-yellow-0-0-1-0)
        (unlocked D-yellow-0-1-0-2)
        (unlocked D-yellow-0-2-1-2)
        (unlocked D-yellow-1-0-2-0)
        (unlocked D-yellow-1-1-1-2)
        (unlocked D-yellow-1-1-2-1)
        (unlocked D-yellow-2-0-2-1)
        (empty-hand)
    )
    (:goal (and
        (at-agent R-2-2))
    )
)
\end{verbatim}

\subsection{\texttt{N-rooms}}
\label{app:nrooms}

The RL environment maintains an
NxN grid with 
the location of rooms, hallways, and walls.
Therefore, a planning state 
obtained from an RL state through the state mapping function
is the name of each room associated with the location.

\subsubsection{PDDL domain}
The PDDL domain file used for the \texttt{N-rooms} problem is described in what follows.
\begin{verbatim}
(define (domain rooms)
  (:requirements :strips :typing)
    (:types
        room - object
    )
    (:predicates 	
        (in-room ?r - room)                         ; Agent current location
        (CONNECTED-ROOMS ?r - room ?s - room)       ; Two rooms are connected
    )
    (:action move-room
        :parameters (?r - room ?s - room)
        :precondition (and
            (CONNECTED-ROOMS ?r ?s)
            (in-room ?r)
        )
        :effect (and
            (not (in-room ?r))
            (in-room ?s)
        )
    )
)
\end{verbatim}

\subsubsection{PDDL instance}
Unlike \texttt{MiniGrid}-based environments,
we randomize both initial and goal locations
on resetting the gym environment.
Therefore, we show one of the auto-generated planning instances
associated with a pair of initial and goal rooms.

\begin{verbatim}
(define (problem rooms-1-16-12__1-2)
    (:domain rooms)
    (:objects
        W c-r0-r2 c-r10-r3 c-r11-r5 c-r2-r1 c-r2-r9 c-r3-r7 
        c-r4-r8 c-r4-r9 c-r6-r3 c-r8-r5 c-r9-r10 
        r0 r1 r10 r11 r2 r3 r4 r5 r6 r7 r8 r9 - room
    )
    (:init
        (CONNECTED-ROOMS c-r0-r2 r0)
        (CONNECTED-ROOMS c-r0-r2 r2)
        (CONNECTED-ROOMS c-r10-r3 r10)
        (CONNECTED-ROOMS c-r10-r3 r3)
        (CONNECTED-ROOMS c-r11-r5 r11)
        (CONNECTED-ROOMS c-r11-r5 r5)
        (CONNECTED-ROOMS c-r2-r1 r1)
        (CONNECTED-ROOMS c-r2-r1 r2)
        (CONNECTED-ROOMS c-r2-r9 r2)
        (CONNECTED-ROOMS c-r2-r9 r9)
        (CONNECTED-ROOMS c-r3-r7 r3)
        (CONNECTED-ROOMS c-r3-r7 r7)
        (CONNECTED-ROOMS c-r4-r8 r4)
        (CONNECTED-ROOMS c-r4-r8 r8)
        (CONNECTED-ROOMS c-r4-r9 r4)
        (CONNECTED-ROOMS c-r4-r9 r9)
        (CONNECTED-ROOMS c-r6-r3 r3)
        (CONNECTED-ROOMS c-r6-r3 r6)
        (CONNECTED-ROOMS c-r8-r5 r5)
        (CONNECTED-ROOMS c-r8-r5 r8)
        (CONNECTED-ROOMS c-r9-r10 r10)
        (CONNECTED-ROOMS c-r9-r10 r9)
        (CONNECTED-ROOMS r0 c-r0-r2)
        (CONNECTED-ROOMS r1 c-r2-r1)
        (CONNECTED-ROOMS r10 c-r10-r3)
        (CONNECTED-ROOMS r10 c-r9-r10)
        (CONNECTED-ROOMS r11 c-r11-r5)
        (CONNECTED-ROOMS r2 c-r0-r2)
        (CONNECTED-ROOMS r2 c-r2-r1)
        (CONNECTED-ROOMS r2 c-r2-r9)
        (CONNECTED-ROOMS r3 c-r10-r3)
        (CONNECTED-ROOMS r3 c-r3-r7)
        (CONNECTED-ROOMS r3 c-r6-r3)
        (CONNECTED-ROOMS r4 c-r4-r8)
        (CONNECTED-ROOMS r4 c-r4-r9)
        (CONNECTED-ROOMS r5 c-r11-r5)
        (CONNECTED-ROOMS r5 c-r8-r5)
        (CONNECTED-ROOMS r6 c-r6-r3)
        (CONNECTED-ROOMS r7 c-r3-r7)
        (CONNECTED-ROOMS r8 c-r4-r8)
        (CONNECTED-ROOMS r8 c-r8-r5)
        (CONNECTED-ROOMS r9 c-r2-r9)
        (CONNECTED-ROOMS r9 c-r4-r9)
        (CONNECTED-ROOMS r9 c-r9-r10)
        (in-room r6)
    )
    (:goal (and
        (in-room r0))
    )
)
\end{verbatim}

\section{Hierarchical Reward Machines}
Hierarchical Reward Machine (HRM)
needs a Finite State Machine (FSM) that
describes the transitions between
symbolic states and events that trigger the transitions.
It can either be written directly or
translated from Linear Temporal Logic (LTL) expressions.
In this paper, 
we defined FSMs
for $\texttt{MiniGrid}$ environments
following the structures
that commonly appear in papers 
based on reward machines.

\textbf{Note that FSMs in 
HRL algorithms based on
LTL/RM encodes knowledge about the solution to the problem.
FSMs must be defined per instance basis,
or a human expert must know a partial solution that 
is general enough so that 
it can be applied to multiple instances.
As the problem domain gets more complicated, 
this manual task is not at all trivial.
}

In this paper, 
we chose hierarchical reward machines (HRM) \cite{icarte2022reward}
as a baseline HRL algorithm
since it is very difficult to find 
a reliable implementation that integrates 
deep RL agents.
While extending the baseline for solving \texttt{MiniGrid} environments,
we defined FSMs similar to the ones in the
baseline method \cite{icarte2022reward}.

\subsection{\texttt{MiniGrid - DoorKey}}
\begin{figure}[h!]
    \centering
    \includegraphics[width=0.7\textwidth]{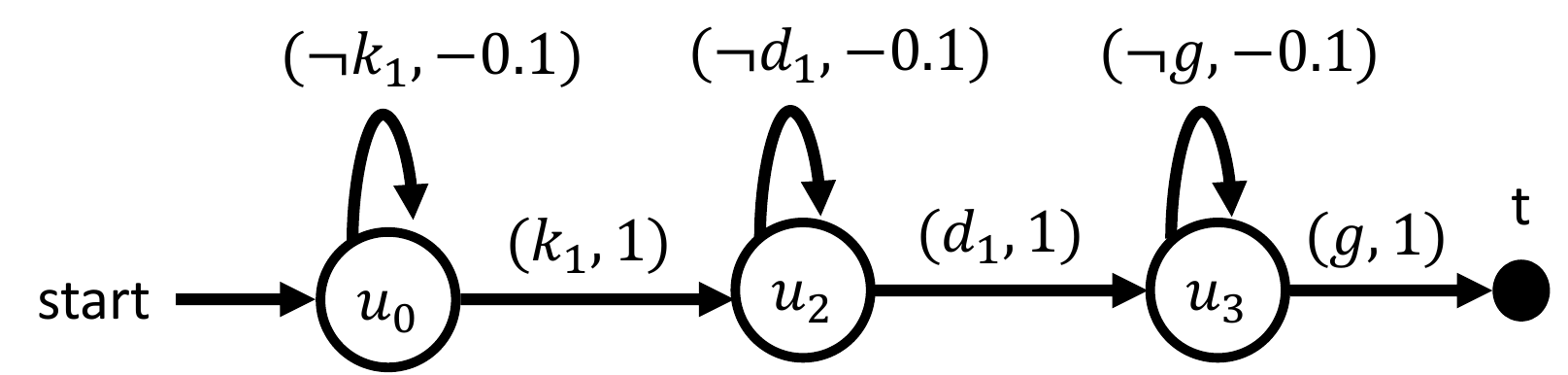}
    \caption{DoorKey}
    \label{fig:app_fig5}
\end{figure}

Nodes $u_0$, $u_1$, and $u_2$ are FSM states.
Upon resetting the RL environment, 
FSM enters the first node $u_0$,
and events defined over the edges
trigger the state transitions.
This reward structure can be used for
defining rewards for the RL environment in a reward machine,
or 
one could define options over FSMs 
that encapsulates temporarily extended actions.
The events are defined as follows:
$k_1$ entails true if the agent picked up the key
at the room,
$d_1$ entails true if the agent was able to
unlock the door connecting two rooms,
$g$ entails true if the agent arrived at the goal room.
Finally, the FSM terminates when 
the agent arrives at the goal tile.
The value next to the event is
the reward that the agent receives.
For example, when the agent was in state $u_0$
and did not pick up the key, then the reward is $-0.1$.
On the other hand, if the agent picked up the key, then
the state transition occurs, and the agent receives a reward of $1$.

\subsection{\texttt{MiniGrid - 4 Rooms with Balls}}
\begin{figure}[h!]
    \centering
    \includegraphics[width=0.7\textwidth]{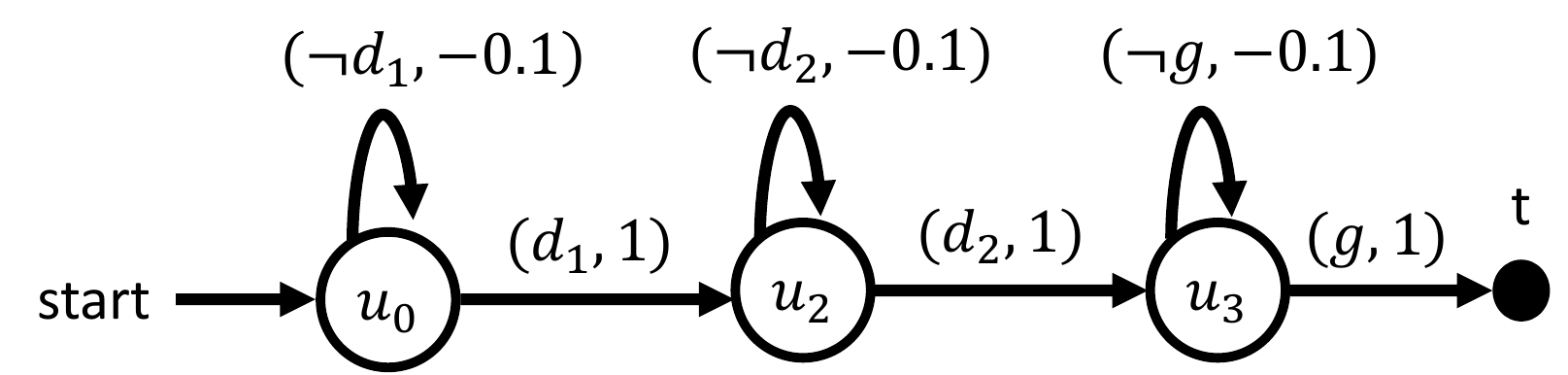}
    \caption{4 Rooms with Balls}
    \label{fig:app_fig6}
\end{figure}
The underlying idea for defining this FSM
is the same for the FSM shown in Figure \ref{fig:app_fig5}.
In this problem domain, there are 3 doors
that connect rooms.
$d_1$ entails true if the agent unlocked 
the room between the upper left and the upper right rooms,
$d_1$ entails true if the agent unlocked
the room between the upper right and the lower right rooms.

\subsection{\texttt{MiniGrid - 4 Rooms with a Locked Door}}
\begin{figure}[h!]
    \centering
    \includegraphics[width=0.7\textwidth]{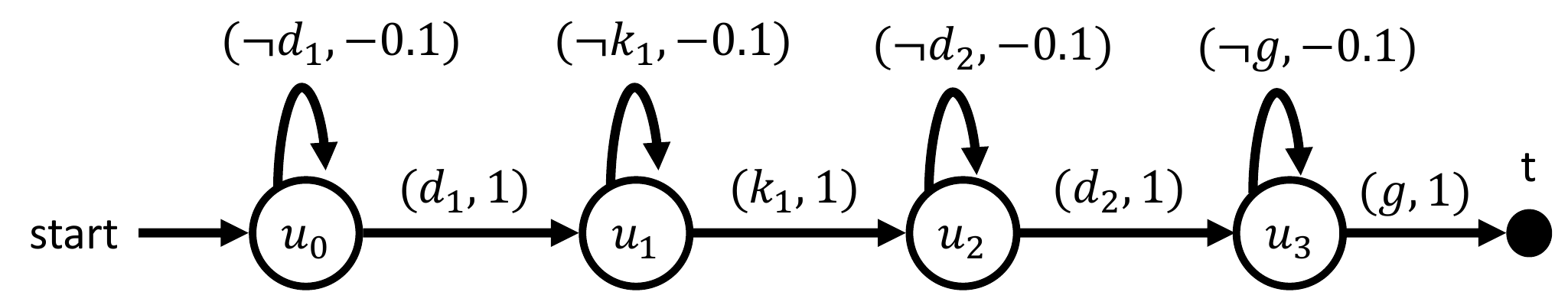}
    \caption{4 Rooms with a Locked Door}
    \label{fig:app_fig7}
\end{figure}
In this problem domain, the agent 
musts use a key to unlocked the goal room.
Therefore, Figure \ref{fig:app_fig7}
encodes such knowledge in the FSM;
from state $u_1$, 
the agent can transit to the next state
if agent picked up the key in the room at the upper right corner ($k_1$).
As we can see, as the solution to the problem becomes
more complex, 
the FSMs has to incorporate such knowledge 
in more complex diagrams, one per each domain.
\textbf{It is worth noting that in order to incorporate the knowledge about solutions in the FSM, one needs first to obtain such knowledge. While for small problems humans can easily spot what a solution is, as problems become more complex, it becomes harder.}

\newpage
\section{Implementation Notes}
In this section, we provide implementation details for
\texttt{HplanPPO}, \texttt{HplanDDQN}, and \texttt{HRM}.
For more additional details,
please refer to the python script code available
in the code supplementary material.

\subsection{Feature Extractors}
For the problem domains generated by \texttt{MiniGrid} environment,
we modified Convolutional Neural Network (CNN) based architecture
presented in \texttt{BabyAI} RL environment \cite{babyai_iclr19}.
The main differences between \texttt{BabyAI} and our \texttt{MiniGrid}-based
gym environments are: 
(1) our experiments are fully-observable,
(2) there's no natural language goal description available in our experiments.

\subsubsection{CNN Feature Extractors for 4 and 9 Rooms Environments}
The CNN feature extractors
first process three-channel input grid into 
the embedding layer since the value at each
grid encodes symbolic state information in integers.
Next, we pass 3-layer CNN, and finally 
we added the last linear layer to the output
the feature vector of size 128.

\begin{verbatim}
class BabyAIFullyObsCNN(BaseFeaturesExtractor):
    def __init__(
        self,
        observation_space: gym.Space,
        features_dim: int = 128,
    ):
        super().__init__(observation_space, features_dim)
        self.max_value = 147
        self.embedding = nn.Embedding(3 * self.max_value, features_dim)
        self.cnn = nn.Sequential(
            nn.Conv2d(in_channels=features_dim, out_channels=features_dim, 
            kernel_size=(3, 3), stride=(2, 2), padding=1),
            nn.BatchNorm2d(features_dim),
            nn.ReLU(),
            nn.Conv2d(in_channels=features_dim, out_channels=features_dim, 
            kernel_size=(3, 3), stride=(2, 2), padding=1),
            nn.BatchNorm2d(features_dim),
            nn.ReLU(),
            nn.Conv2d(in_channels=features_dim, out_channels=features_dim, 
            kernel_size=(3, 3), stride=(2, 2), padding=1),
            nn.BatchNorm2d(features_dim),
            nn.ReLU(),
            nn.MaxPool2d(kernel_size=(2, 2), stride=(2, 2), padding=0),
            nn.Flatten()
        )
        self.linear = nn.Sequential(
            nn.Linear(n_flatten, features_dim),
            nn.ReLU()
        )
        self.apply(initialize_parameters)

    def forward(self, observations: th.Tensor):        
        offsets = th.Tensor([0, self.max_value, 2 * self.max_value])
        x = (observations + offsets[None, :, None, None]).long()
        x =  self.embedding(x).sum(1).permute(0, 3, 1, 2)
        x = self.cnn(x)
        x = self.linear(x)
        return x
\end{verbatim}

\subsubsection{CNN Feature Extractors for Door Key environment}
The architecture remains the same as above
except for the CNN only has two layers 
when the input dimension becomes smaller.
In addition to processing
the feature values in the grid,
the following code snippet also shows
the option labels will also be concatenated 
with the feature vector after passing 
an embedding layer and one additional linear layer.
These option label features are necessary 
for implementing \texttt{DDQN}-based algorithms.

\begin{verbatim}
class BabyAIFullyObsSmallCNNDict(BaseFeaturesExtractor):
    def __init__(
            self,
            observation_space: gym.Space,
            features_dim: int = 128,
    ):
        super().__init__(observation_space, features_dim)
        image_observation_space = observation_space.spaces['image']

        self.max_value = 147
        self.embedding = nn.Embedding(3 * self.max_value, features_dim)
        self.cnn = nn.Sequential(
            nn.Conv2d(in_channels=features_dim, out_channels=features_dim, 
            kernel_size=(3, 3), stride=(2, 2), padding=1),
            nn.BatchNorm2d(features_dim),
            nn.ReLU(),
            nn.Conv2d(in_channels=features_dim, out_channels=features_dim, 
            kernel_size=(3, 3), stride=(2, 2), padding=1),
            nn.BatchNorm2d(features_dim),
            nn.ReLU(),
            nn.MaxPool2d(kernel_size=(2, 2), stride=(2, 2), padding=0),
            nn.Flatten()
        )
        self.linear = nn.Sequential(
            nn.Linear(n_flatten, features_dim),
            nn.ReLU()
        )
        label_observation_space = observation_space.spaces['label']
        self.label_embedding = nn.Linear(label_observation_space.n, features_dim)
        self.linear2 = nn.Sequential(
            nn.Linear(features_dim * 2, features_dim),
            nn.ReLU()
        )
        self.apply(initialize_parameters)  # Initialize parameters correctly

    def forward(self, observations: th.Tensor) -> th.Tensor:
        x = observations['image']
        offsets = th.Tensor([0, self.max_value, 2 * self.max_value]).to(x.device)
        x = (x + offsets[None, :, None, None]).long()
        x = self.embedding(x).sum(1).permute(0, 3, 1, 2)
        x = self.cnn(x)
        x = self.linear(x)
        y = observations['label']      
        y = th.squeeze(y)
        y = self.label_embedding(y)
        if y.ndim == 1:
            y = y.reshape((1, -1))
        z = th.cat((x,y), dim=1)
        z = self.linear2(z)
        return z
\end{verbatim}

\subsubsection{Flattening Observation}
Lastly, we also tested a features extractor
that flattens the input observations 
in 2-D arrays into a single 1-D array 
with one-hot encoding.
This flattening observation feature extractors
are used in \texttt{N-rooms} environments and
the original code implementation in hierarchical reward machines \cite{icarte2022reward}.

\subsection{PPO Hyperparameters}

\paragraph{\texttt{MiniGrid}-based Environments}

\begin{itemize}
\item learning rate: $2.5e^{-4}$
\item clip: 0.2
\item option networks for policy and value: $[64, 64]$
\item flat RL networks for policy and value: $[128, 128]$
\item number of steps per rollout: 2048
\item batch size: 256
\item number of epochs for PPO update: 10
\item gamma: 0.99
\item lambda for GAE: 0.95
\item entropy coefficient: 0.01
\item value function coefficient: 0.5
\item maximum gradient norm: 0.05
\item maximum episode length: 1024
\item option termination reward: +1
\item option unit penalty cost: $\frac{0.9}{1024}$
\end{itemize}

\paragraph{\texttt{N-rooms}-based Environments}
\begin{itemize}
\item learning rate: $1e^{-4}$
\item clip: 0.1
\item networks for policy and value: $[256, 256]$
\item number of steps per rollout: 2048
\item batch size: 128
\item number of epochs for PPO update: 10
\item gamma: 0.99
\item lambda for GAE: 0.95
\item entropy coefficient: 0.01
\item value function coefficient: 0.5
\item maximum gradient norm: 0.05
\item maximum episode length: 1024
\item option termination reward: +1
\item option unit penalty cost: -0.05
\end{itemize}

\subsection{DDQN Hyperparameters}

\paragraph{\texttt{MiniGrid}-based Environments}
\begin{itemize}
    \item replay buffer size: 149504
    \item learning rate: 0.0005
    \item learning starts: 1000
    \item batch size: 96
    \item tau: 1.0
    \item gamma: 0.9
    \item train frequency: 1 step
    \item gradient steps: 1
    \item target update interval: 500
    \item exploration fraction: 0.1
    \item exploration initial epsilon: 1.0
    \item exploration final epsilon: 0.1,
    \item max gradient norm: 10
    \item option networks for value: $[64, 64]$
    \item flat RL networks for value: $[128, 128]$
\end{itemize}
The DDQN hyperparameters are chosen from
the baseline agent implementation \cite{icarte2022reward},

\newpage
\section{More Results}

\subsection{\texttt{MiniGrid DoorKey}}
Figure \ref{fig:options_doorkey} shows 
learning progress of options and 
we show a shortest plan from annotated planning task.
\begin{figure}[h!]
\centering
\begin{subfigure}[b]{0.34\textwidth}
  \includegraphics[width=\textwidth]{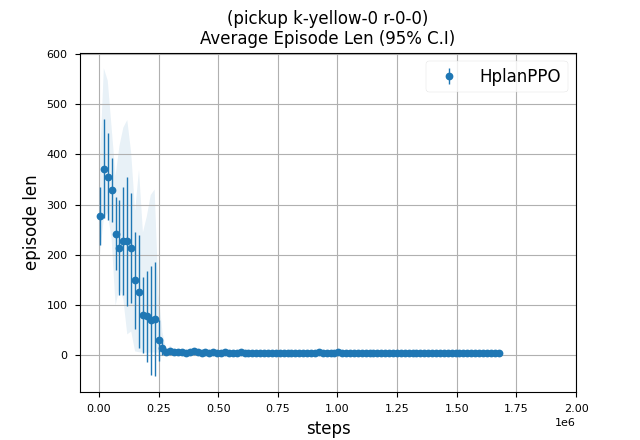}
  \caption{(pickup k-yellow-0 r-0-0) Episode Length}
\end{subfigure}
\begin{subfigure}[b]{0.34\textwidth}
  \includegraphics[width=\textwidth]{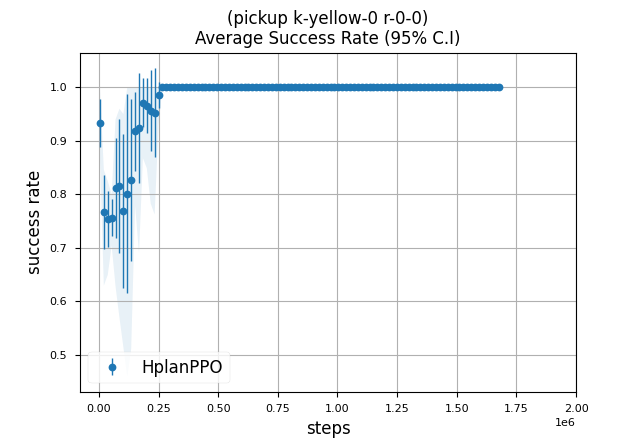}
  \caption{(pickup k-yellow-0 r-0-0) Success Rate}
\end{subfigure}
\begin{subfigure}[b]{0.34\textwidth}
  \includegraphics[width=\textwidth]{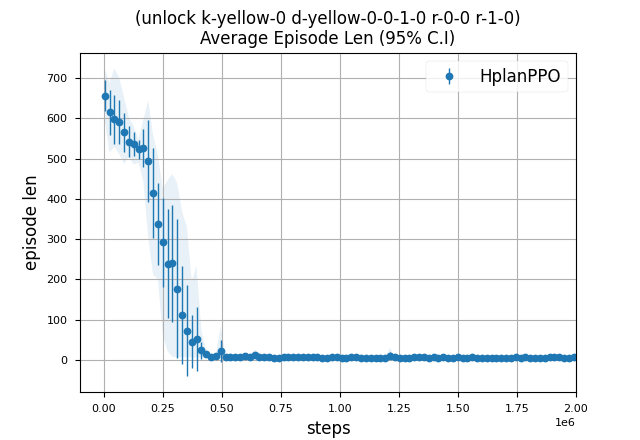}
  \caption{unlock k-yellow-0 d-yellow-0-0-1-0 r-0-0 r-1-0) Episode Length}
\end{subfigure}
\begin{subfigure}[b]{0.34\textwidth}
  \includegraphics[width=\textwidth]{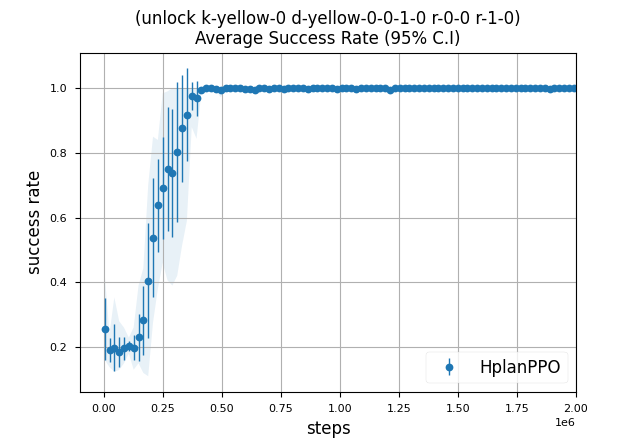}
  \caption{unlock k-yellow-0 d-yellow-0-0-1-0 r-0-0 r-1-0) Success Rate}
\end{subfigure}
\begin{subfigure}[b]{0.34\textwidth}
  \includegraphics[width=\textwidth]{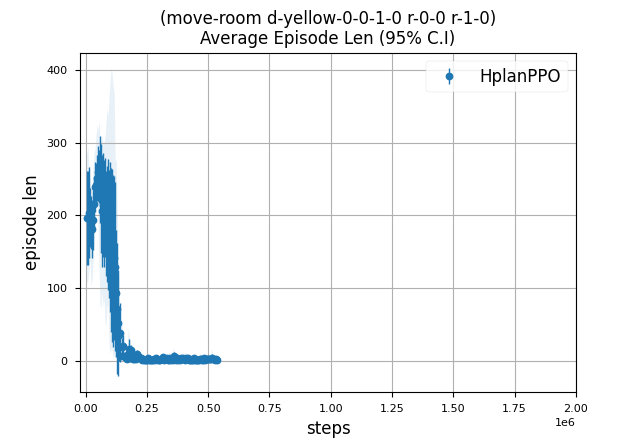}
  \caption{(move-room d-yellow-0-0-1-0 r-0-0 r-1-0) Episode Length}
\end{subfigure}
\begin{subfigure}[b]{0.34\textwidth}
  \includegraphics[width=\textwidth]{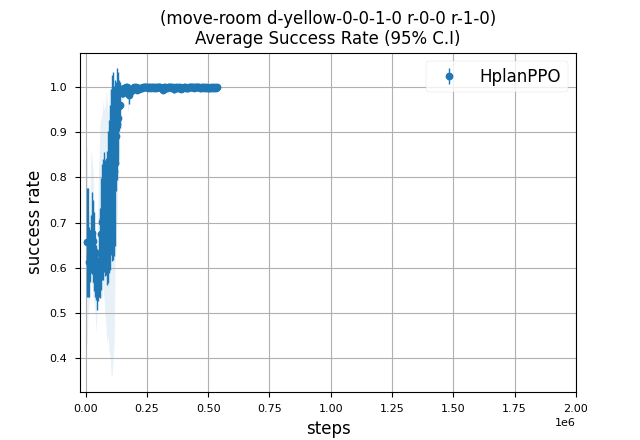}
  \caption{(move-room d-yellow-0-0-1-0 r-0-0 r-1-0) Success Rate}
\end{subfigure}
\begin{subfigure}[b]{0.34\textwidth}
  \includegraphics[width=\textwidth]{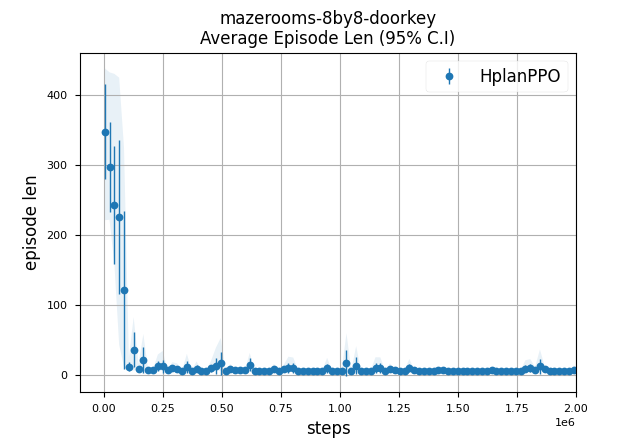}
  \caption{Goal option Episode Length}
\end{subfigure}
\begin{subfigure}[b]{0.34\textwidth}
  \includegraphics[width=\textwidth]{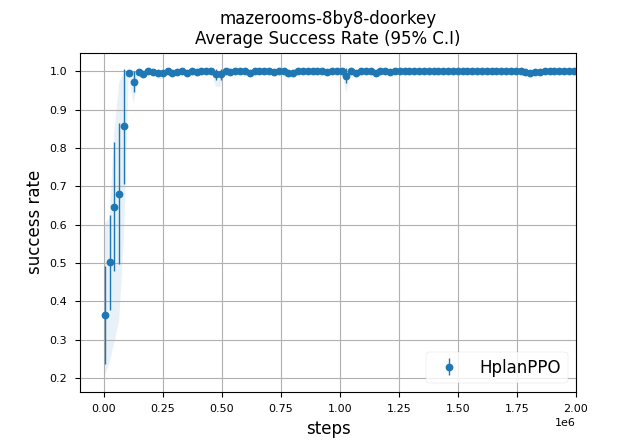}
  \caption{Goal option Success Rate}
\end{subfigure}
\caption{
Option Learning Progress in \texttt{MiniGrid Door Key} Domain
}
\label{fig:options_doorkey}
\end{figure}

\paragraph{Plan from Annotated Planning Task}
\begin{verbatim}
state:0
(locked d-yellow-0-0-1-0)
(at-agent r-0-0)
(empty-hand)
(at k-yellow-0 r-0-0)

action:0
(pickup k-yellow-0 r-0-0)
  PRE: (at k-yellow-0 r-0-0)
  PRE: (empty-hand)
  PRE: (at-agent r-0-0)
  ADD: (carry k-yellow-0)
  DEL: (at k-yellow-0 r-0-0)
  DEL: (empty-hand)

state:1
(carry k-yellow-0)
(locked d-yellow-0-0-1-0)
(at-agent r-0-0)

action:1
(unlock k-yellow-0 d-yellow-0-0-1-0 r-0-0 r-1-0)
  PRE: (carry k-yellow-0)
  PRE: (locked d-yellow-0-0-1-0)
  PRE: (at-agent r-0-0)
  ADD: (unlocked d-yellow-0-0-1-0)
  DEL: (locked d-yellow-0-0-1-0)

state:2
(carry k-yellow-0)
(unlocked d-yellow-0-0-1-0)
(at-agent r-0-0)

action:2
(move-room d-yellow-0-0-1-0 r-0-0 r-1-0)
  PRE: (unlocked d-yellow-0-0-1-0)
  PRE: (at-agent r-0-0)
  ADD: (at-agent r-1-0)
  DEL: (at-agent r-0-0)
\end{verbatim}


\newpage
\subsection{\texttt{MiniGrid 4 Rooms with a Locked Door}}
Figures \ref{fig:options_4rooms_a} -- \ref{fig:options_4rooms_b} show
learning progress of options and 
we show a shortest plan from annotated planning task.
\begin{figure}[h!]
\centering
\begin{subfigure}[b]{0.34\textwidth}
  \includegraphics[width=\textwidth]{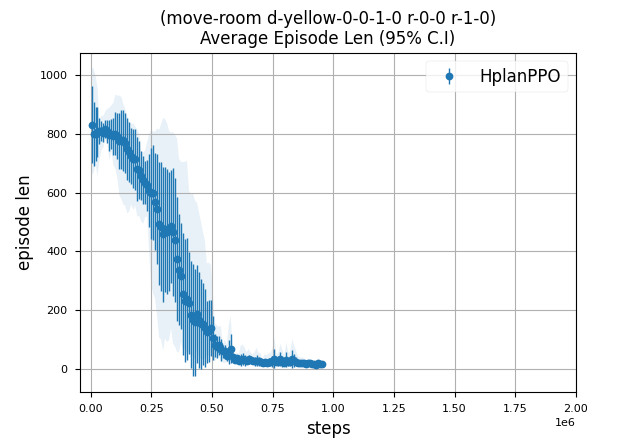}
  \caption{(move-room d-yellow-0-0-1-0 r-0-0 r-1-0) Episode Length}
\end{subfigure}
\begin{subfigure}[b]{0.34\textwidth}
  \includegraphics[width=\textwidth]{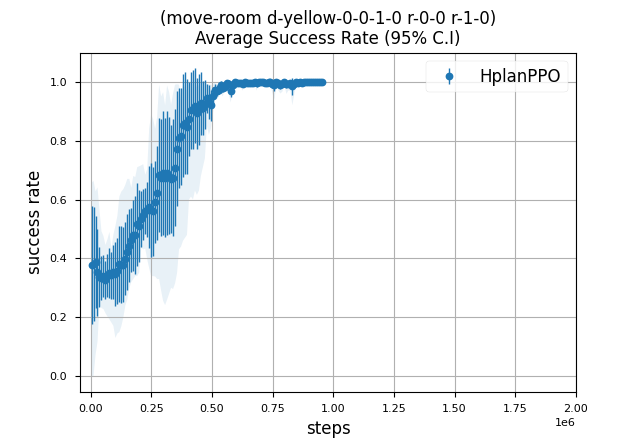}
  \caption{(move-room d-yellow-0-0-1-0 r-0-0 r-1-0) Success Rate}
\end{subfigure}
\begin{subfigure}[b]{0.34\textwidth}
  \includegraphics[width=\textwidth]{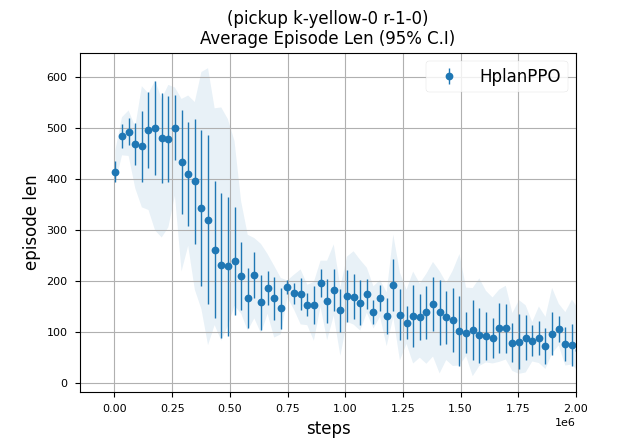}
  \caption{(pickup k-yellow-0 r-1-0) Episode Length}
\end{subfigure}
\begin{subfigure}[b]{0.34\textwidth}
  \includegraphics[width=\textwidth]{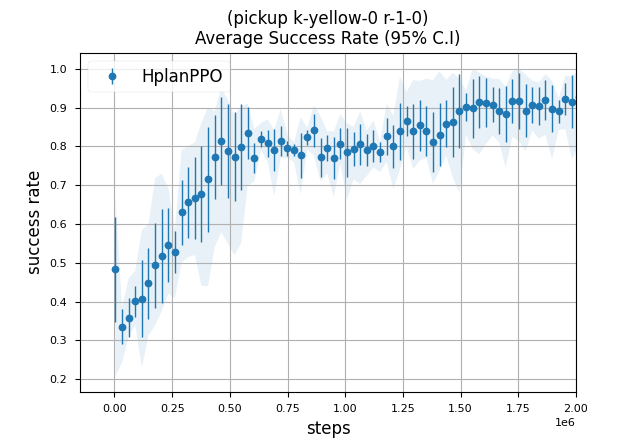}
  \caption{(pickup k-yellow-0 r-1-0) Success Rate}
\end{subfigure}
\begin{subfigure}[b]{0.34\textwidth}
  \includegraphics[width=\textwidth]{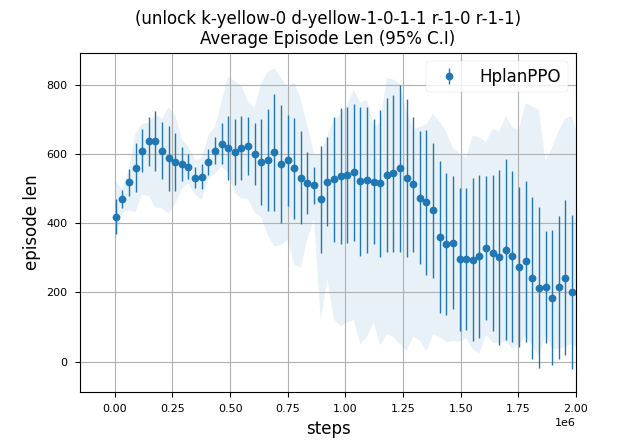}
  \caption{(unlock k-yellow-0 d-yellow-1-0-1-1 r-1-0 r-1-1) Episode Length}
\end{subfigure}
\begin{subfigure}[b]{0.34\textwidth}
  \includegraphics[width=\textwidth]{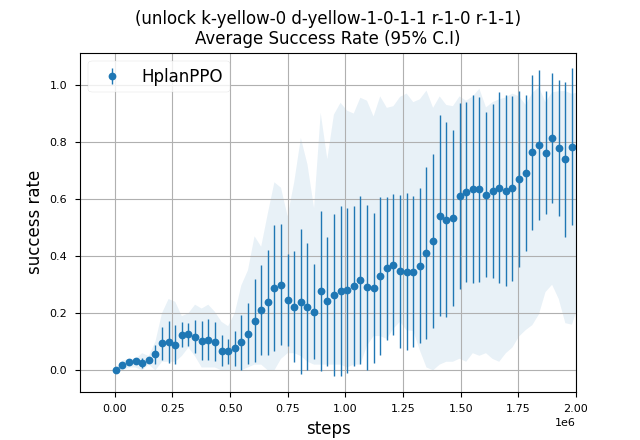}
  \caption{(unlock k-yellow-0 d-yellow-1-0-1-1 r-1-0 r-1-1) Success Rate}
\end{subfigure}
\begin{subfigure}[b]{0.34\textwidth}
  \includegraphics[width=\textwidth]{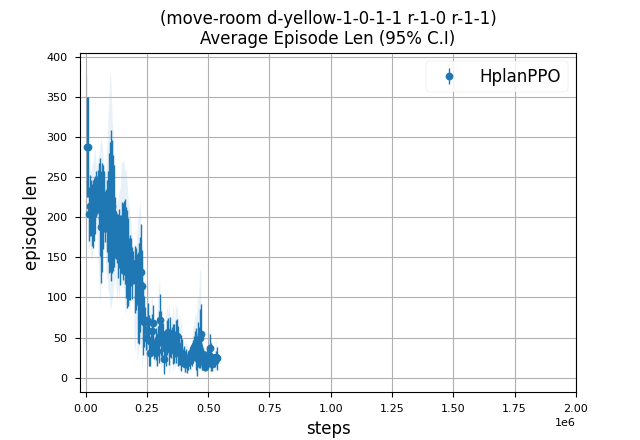}
  \caption{(move-room d-yellow-1-0-1-1 r-1-0 r-1-1) Episode Length}
\end{subfigure}
\begin{subfigure}[b]{0.34\textwidth}
  \includegraphics[width=\textwidth]{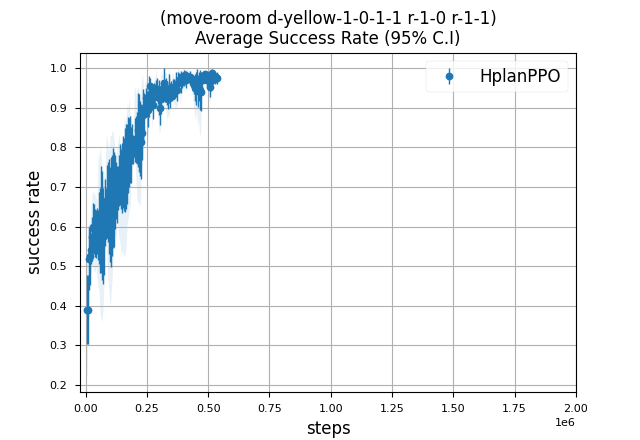}
  \caption{(move-room d-yellow-1-0-1-1 r-1-0 r-1-1) Success Rate}
\end{subfigure}
\caption{
Option Learning Progress in \texttt{MiniGrid 4 Rooms with a Locked Door} Domain
}
\label{fig:options_4rooms_a}
\end{figure}

\begin{figure}[h!]
\centering
\begin{subfigure}[b]{0.34\textwidth}
  \includegraphics[width=\textwidth]{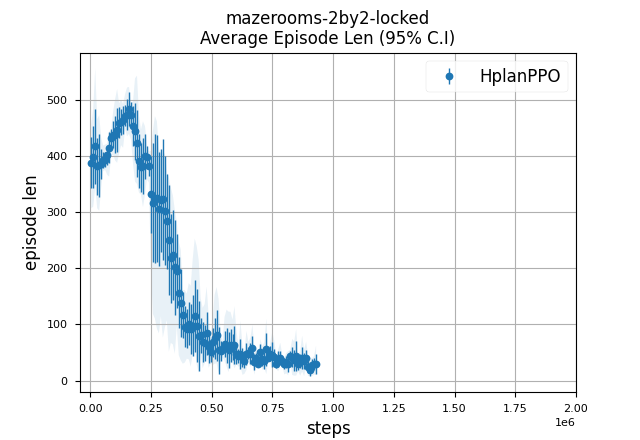}
  \caption{Goal Option Episode Length}
\end{subfigure}
\begin{subfigure}[b]{0.34\textwidth}
  \includegraphics[width=\textwidth]{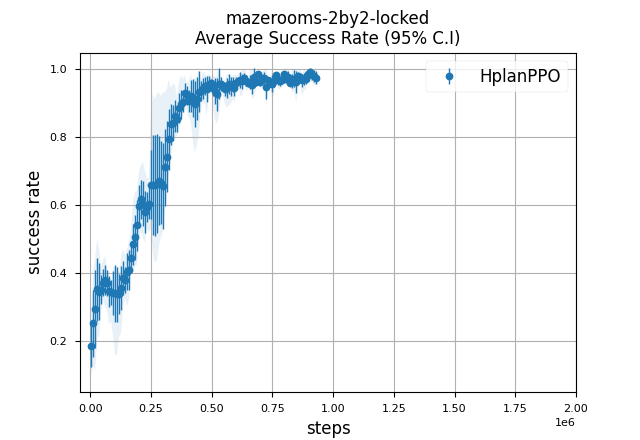}
  \caption{Goal Option Success Rate}
\end{subfigure}
\caption{
Option Learning Progress in \texttt{MiniGrid 4 Rooms with a Locked Door} Domain
}
\label{fig:options_4rooms_b}
\end{figure}

\paragraph{Plan from Annotated Planning Task}
\begin{verbatim}
state:0
(unlocked d-yellow-0-0-0-1)
(unlocked d-yellow-0-0-1-0)
(locked d-yellow-1-0-1-1)
(empty-hand)
(at k-yellow-0 r-1-0)
(at-agent r-0-0)

action:0
(move-room d-yellow-0-0-1-0 r-0-0 r-1-0)
  PRE: (unlocked d-yellow-0-0-1-0)
  PRE: (at-agent r-0-0)
  ADD: (at-agent r-1-0)
  DEL: (at-agent r-0-0)

state:1
(empty-hand)
(unlocked d-yellow-0-0-0-1)
(locked d-yellow-1-0-1-1)
(at-agent r-1-0)
(at k-yellow-0 r-1-0)
(unlocked d-yellow-0-0-1-0)

action:1
(pickup k-yellow-0 r-1-0)
  PRE: (at-agent r-1-0)
  PRE: (at k-yellow-0 r-1-0)
  PRE: (empty-hand)
  ADD: (carry k-yellow-0)
  DEL: (at k-yellow-0 r-1-0)
  DEL: (empty-hand)

state:2
(carry k-yellow-0)
(unlocked d-yellow-0-0-0-1)
(locked d-yellow-1-0-1-1)
(at-agent r-1-0)
(unlocked d-yellow-0-0-1-0)

action:2
(unlock k-yellow-0 d-yellow-1-0-1-1 r-1-0 r-1-1)
  PRE: (at-agent r-1-0)
  PRE: (carry k-yellow-0)
  PRE: (locked d-yellow-1-0-1-1)
  ADD: (unlocked d-yellow-1-0-1-1)
  DEL: (locked d-yellow-1-0-1-1)

state:3
(carry k-yellow-0)
(unlocked d-yellow-1-0-1-1)
(unlocked d-yellow-0-0-0-1)
(at-agent r-1-0)
(unlocked d-yellow-0-0-1-0)

action:3
(move-room d-yellow-1-0-1-1 r-1-0 r-1-1)
  PRE: (at-agent r-1-0)
  PRE: (unlocked d-yellow-1-0-1-1)
  ADD: (at-agent r-1-1)
  DEL: (at-agent r-1-0)
\end{verbatim}

\newpage
\subsection{\texttt{MiniGrid 9 Rooms with Locked Doors}}
Figure \ref{fig:options_9rooms_a} -- \ref{fig:options_9rooms_b} show
learning progress of options and 
we show a shortest plan from annotated planning task.
\begin{figure}[h!]
\centering
\begin{subfigure}[b]{0.34\textwidth}
  \includegraphics[width=\textwidth]{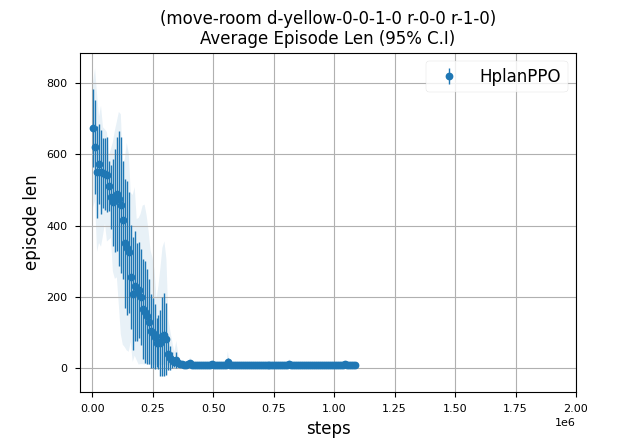}
  \caption{(move-room d-yellow-0-0-1-0 r-0-0 r-1-0) Episode Length}
\end{subfigure}
\begin{subfigure}[b]{0.34\textwidth}
  \includegraphics[width=\textwidth]{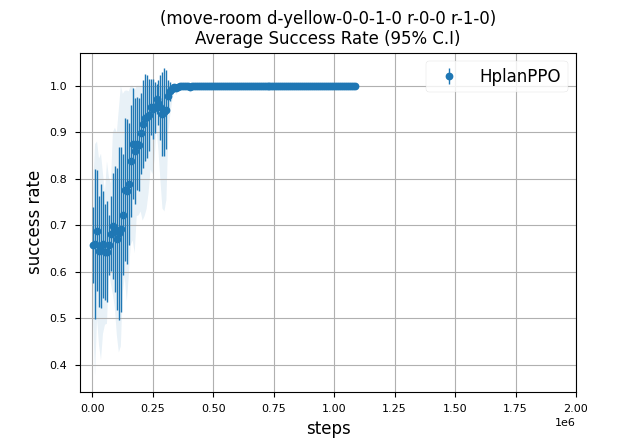}
  \caption{(move-room d-yellow-0-0-1-0 r-0-0 r-1-0) Success Rate}
\end{subfigure}
\begin{subfigure}[b]{0.34\textwidth}
  \includegraphics[width=\textwidth]{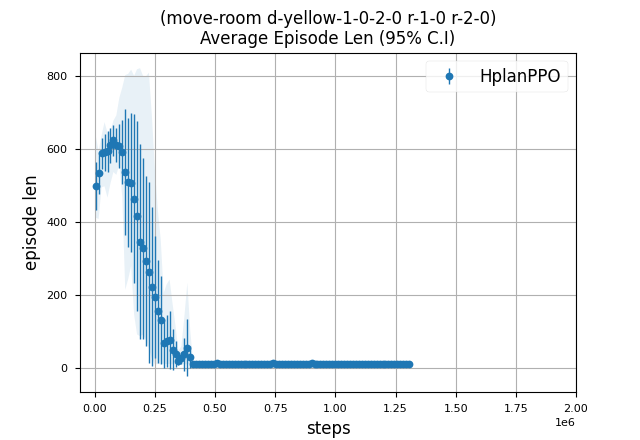}
  \caption{(move-room d-yellow-1-0-2-0 r-1-0 r-2-0) Episode Length}
\end{subfigure}
\begin{subfigure}[b]{0.34\textwidth}
  \includegraphics[width=\textwidth]{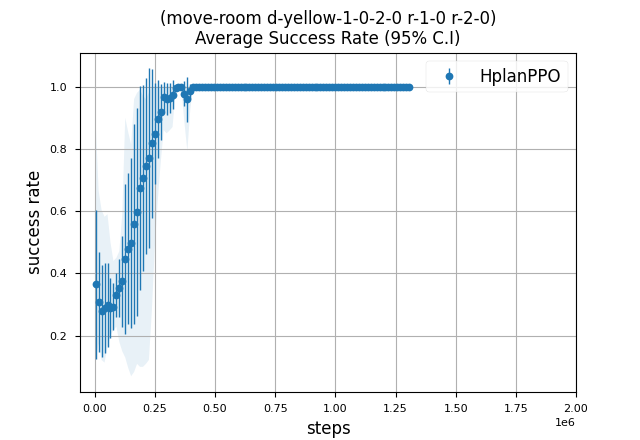}
  \caption{(move-room d-yellow-1-0-2-0 r-1-0 r-2-0) Success Rate}
\end{subfigure}
\begin{subfigure}[b]{0.34\textwidth}
  \includegraphics[width=\textwidth]{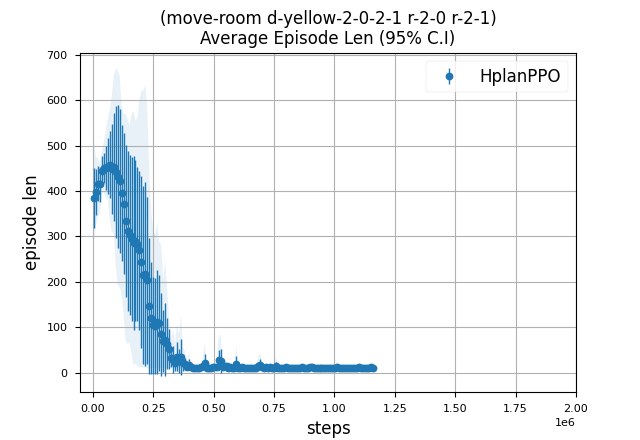}
  \caption{(move-room d-yellow-2-0-2-1 r-2-0 r-2-1) Episode Length}
\end{subfigure}
\begin{subfigure}[b]{0.34\textwidth}
  \includegraphics[width=\textwidth]{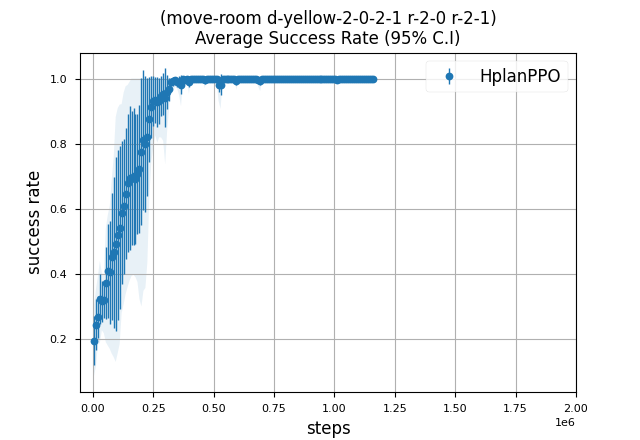}
  \caption{(move-room d-yellow-2-0-2-1 r-2-0 r-2-1) Success Rate}
\end{subfigure}
\begin{subfigure}[b]{0.34\textwidth}
  \includegraphics[width=\textwidth]{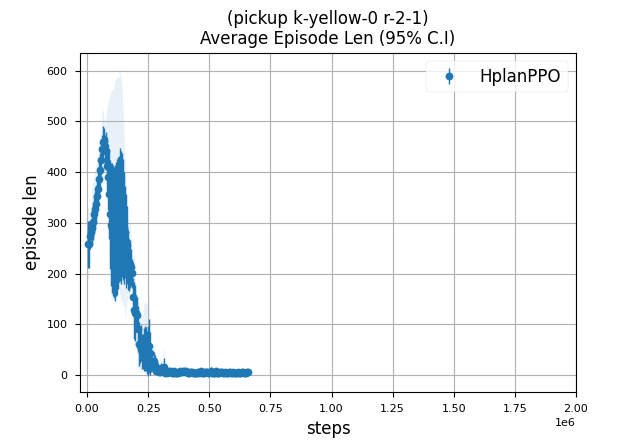}
  \caption{(pickup k-yellow-0 r-2-1) Episode Length}
\end{subfigure}
\begin{subfigure}[b]{0.34\textwidth}
  \includegraphics[width=\textwidth]{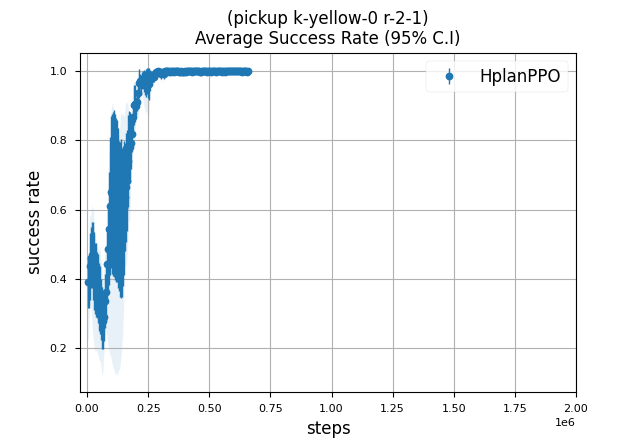}
  \caption{(pickup k-yellow-0 r-2-1) Success Rate}
\end{subfigure}
\caption{
Option Learning Progress in \texttt{MiniGrid 9 Rooms with Locked Doors} Domain
}
\label{fig:options_9rooms_a}
\end{figure}


\begin{figure}[h!]
\centering
\begin{subfigure}[b]{0.34\textwidth}
  \includegraphics[width=\textwidth]{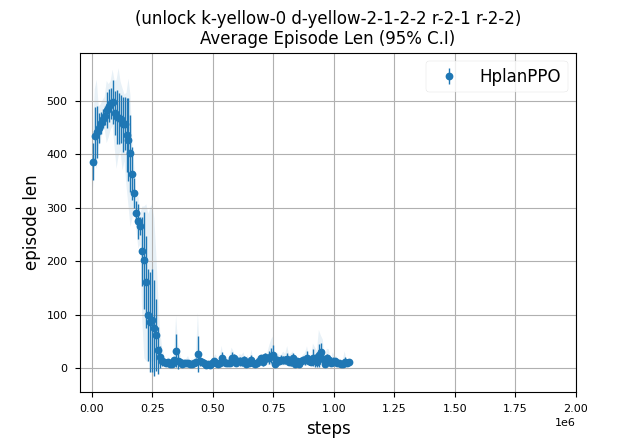}
  \caption{(unlock k-yellow-0 d-yellow-2-1-2-2 r-2-1 r-2-2) Episode Length}
\end{subfigure}
\begin{subfigure}[b]{0.34\textwidth}
  \includegraphics[width=\textwidth]{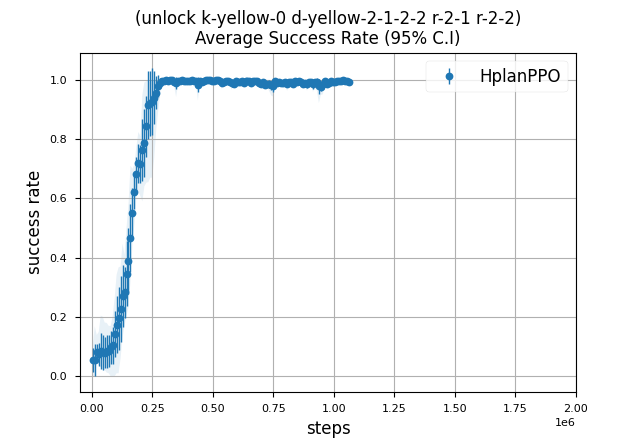}
  \caption{(unlock k-yellow-0 d-yellow-2-1-2-2 r-2-1 r-2-2) Success Rate}
\end{subfigure}
\begin{subfigure}[b]{0.34\textwidth}
  \includegraphics[width=\textwidth]{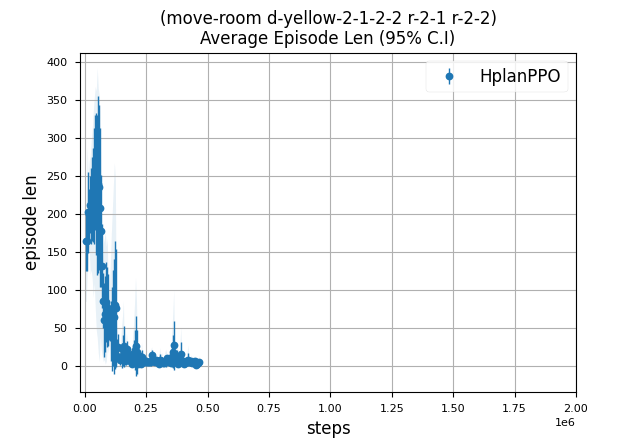}
  \caption{(move-room d-yellow-2-1-2-2 r-2-1 r-2-2) Episode Length}
\end{subfigure}
\begin{subfigure}[b]{0.34\textwidth}
  \includegraphics[width=\textwidth]{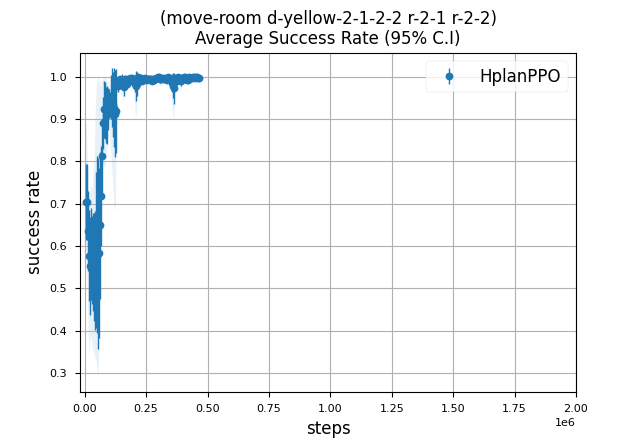}
  \caption{(move-room d-yellow-2-1-2-2 r-2-1 r-2-2) Success Rate}
\end{subfigure}
\begin{subfigure}[b]{0.34\textwidth}
  \includegraphics[width=\textwidth]{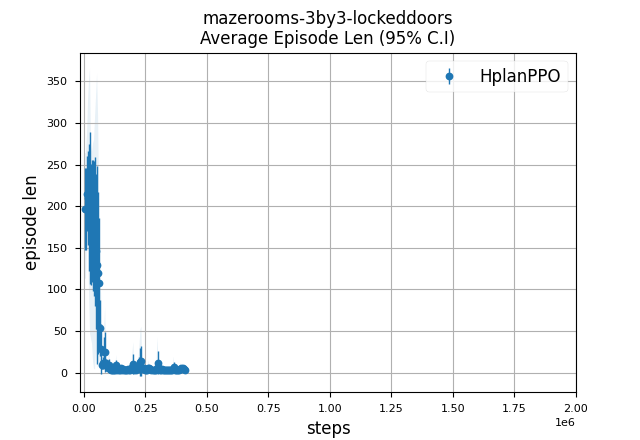}
  \caption{Goal Option Episode Length}
\end{subfigure}
\begin{subfigure}[b]{0.34\textwidth}
  \includegraphics[width=\textwidth]{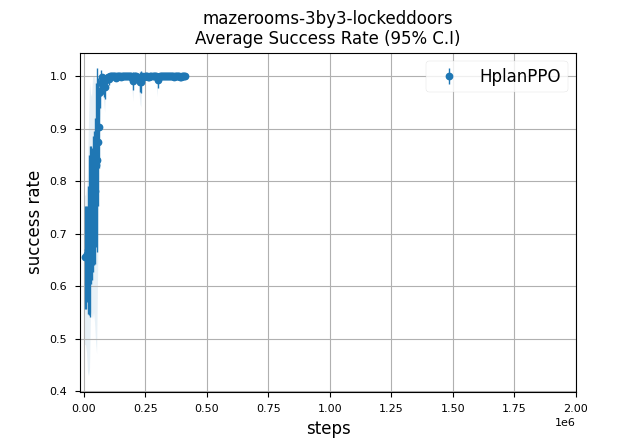}
  \caption{Goal Option Success Rate}
\end{subfigure}
\caption{
Option Learning Progress in \texttt{MiniGrid 9 Rooms with Locked Doors} Domain
}
\label{fig:options_9rooms_b}
\end{figure}

\paragraph{Plan from Annotated Planning Task}
\begin{verbatim}
state:0
(unlocked d-yellow-1-1-1-2)
(locked d-yellow-1-2-2-2)
(unlocked d-yellow-2-0-2-1)
(unlocked d-yellow-1-0-2-0)
(unlocked d-yellow-0-0-0-1)
(at k-yellow-0 r-2-1)
(locked d-yellow-0-1-1-1)
(unlocked d-yellow-0-0-1-0)
(locked d-yellow-1-0-1-1)
(unlocked d-yellow-0-1-0-2)
(unlocked d-yellow-0-2-1-2)
(locked d-yellow-2-1-2-2)
(empty-hand)
(unlocked d-yellow-1-1-2-1)
(at-agent r-0-0)

action:0
(move-room d-yellow-0-0-1-0 r-0-0 r-1-0)
  PRE: (unlocked d-yellow-0-0-1-0)
  PRE: (at-agent r-0-0)
  ADD: (at-agent r-1-0)
  DEL: (at-agent r-0-0)

state:1
(unlocked d-yellow-1-1-1-2)
(locked d-yellow-1-2-2-2)
(unlocked d-yellow-2-0-2-1)
(unlocked d-yellow-1-0-2-0)
(unlocked d-yellow-0-0-0-1)
(at-agent r-1-0)
(at k-yellow-0 r-2-1)
(locked d-yellow-0-1-1-1)
(unlocked d-yellow-0-0-1-0)
(unlocked d-yellow-0-1-0-2)
(unlocked d-yellow-0-2-1-2)
(locked d-yellow-2-1-2-2)
(empty-hand)
(locked d-yellow-1-0-1-1)
(unlocked d-yellow-1-1-2-1)

action:1
(move-room d-yellow-1-0-2-0 r-1-0 r-2-0)
  PRE: (at-agent r-1-0)
  PRE: (unlocked d-yellow-1-0-2-0)
  ADD: (at-agent r-2-0)
  DEL: (at-agent r-1-0)

state:2
(unlocked d-yellow-1-1-1-2)
(locked d-yellow-1-2-2-2)
(unlocked d-yellow-2-0-2-1)
(unlocked d-yellow-1-0-2-0)
(unlocked d-yellow-0-0-0-1)
(at k-yellow-0 r-2-1)
(locked d-yellow-0-1-1-1)
(unlocked d-yellow-0-0-1-0)
(locked d-yellow-1-0-1-1)
(unlocked d-yellow-0-1-0-2)
(unlocked d-yellow-0-2-1-2)
(locked d-yellow-2-1-2-2)
(at-agent r-2-0)
(empty-hand)
(unlocked d-yellow-1-1-2-1)

action:2
(move-room d-yellow-2-0-2-1 r-2-0 r-2-1)
  PRE: (unlocked d-yellow-2-0-2-1)
  PRE: (at-agent r-2-0)
  ADD: (at-agent r-2-1)
  DEL: (at-agent r-2-0)

state:3
(unlocked d-yellow-1-1-1-2)
(locked d-yellow-1-2-2-2)
(unlocked d-yellow-2-0-2-1)
(unlocked d-yellow-1-0-2-0)
(unlocked d-yellow-0-0-0-1)
(at-agent r-2-1)
(at k-yellow-0 r-2-1)
(locked d-yellow-0-1-1-1)
(unlocked d-yellow-0-0-1-0)
(unlocked d-yellow-0-1-0-2)
(unlocked d-yellow-0-2-1-2)
(locked d-yellow-2-1-2-2)
(empty-hand)
(locked d-yellow-1-0-1-1)
(unlocked d-yellow-1-1-2-1)

action:3
(pickup k-yellow-0 r-2-1)
  PRE: (at k-yellow-0 r-2-1)
  PRE: (empty-hand)
  PRE: (at-agent r-2-1)
  ADD: (carry k-yellow-0)
  DEL: (at k-yellow-0 r-2-1)
  DEL: (empty-hand)

state:4
(carry k-yellow-0)
(unlocked d-yellow-1-1-1-2)
(locked d-yellow-1-2-2-2)
(unlocked d-yellow-2-0-2-1)
(unlocked d-yellow-1-0-2-0)
(unlocked d-yellow-0-0-0-1)
(at-agent r-2-1)
(locked d-yellow-0-1-1-1)
(unlocked d-yellow-0-0-1-0)
(unlocked d-yellow-0-1-0-2)
(unlocked d-yellow-0-2-1-2)
(locked d-yellow-2-1-2-2)
(locked d-yellow-1-0-1-1)
(unlocked d-yellow-1-1-2-1)

action:4
(unlock k-yellow-0 d-yellow-2-1-2-2 r-2-1 r-2-2)
  PRE: (carry k-yellow-0)
  PRE: (locked d-yellow-2-1-2-2)
  PRE: (at-agent r-2-1)
  ADD: (unlocked d-yellow-2-1-2-2)
  DEL: (locked d-yellow-2-1-2-2)

state:5
(carry k-yellow-0)
(unlocked d-yellow-1-1-1-2)
(locked d-yellow-1-2-2-2)
(unlocked d-yellow-2-0-2-1)
(unlocked d-yellow-1-0-2-0)
(unlocked d-yellow-0-0-0-1)
(at-agent r-2-1)
(locked d-yellow-0-1-1-1)
(unlocked d-yellow-0-0-1-0)
(unlocked d-yellow-2-1-2-2)
(unlocked d-yellow-0-1-0-2)
(unlocked d-yellow-0-2-1-2)
(locked d-yellow-1-0-1-1)
(unlocked d-yellow-1-1-2-1)

action:5
(move-room d-yellow-2-1-2-2 r-2-1 r-2-2)
  PRE: (unlocked d-yellow-2-1-2-2)
  PRE: (at-agent r-2-1)
  ADD: (at-agent r-2-2)
  DEL: (at-agent r-2-1)
\end{verbatim}

\subsection{Other Algorithms}
For other algorithms \texttt{HplanDDQN} and \texttt{HRM},
we tried hyperparameter tuning for \texttt{DDQN}
and tested with CNN and Flattened observation feature extractors.
\texttt{HRM} was able to solve \texttt{MiniGrid DoorKey} environment,
but it couldn't solve larger domains.
\texttt{HplanDDQN} was not able to solve any \texttt{MiniGrid}-based domains.
Both \texttt{HRM} and \texttt{HplanDDQN}
have the same neural network architecture,
but the main difference between \texttt{HRM} and \texttt{HplanDDQN}
is that \texttt{HplanDDQN} did not reuse samples across options.

\end{document}